\documentclass[a4paper,fleqn]{cas-dc}
\usepackage[numbers]{natbib}

\usepackage{graphicx, subcaption}
\usepackage{url}
\usepackage[textsize=tiny]{todonotes}

\usepackage{placeins} %
\usepackage{ntheorem}
\usepackage{amsmath,amssymb,mathtools}
\usepackage[algo2e,noend,linesnumbered,ruled]{algorithm2e}\DontPrintSemicolon
\usepackage{etoolbox}
\makeatletter\patchcmd{\@algocf@start}{-1.5em}{0pt}{}{}\makeatother

\usepackage{tabto}
\usepackage{pgfplots}
\newtheorem{axiom}{Axiom}
\newtheorem{definition}{Definition}
\newtheorem{theorem}{Theorem}
\newcommand{\nearest}{{n_1}}
\newcommand{\second}{{n_2}}

\newcommand{\MS}{\ensuremath{\mathrm{\tilde{S}}}}
\newcommand{\ms}{\ensuremath{\mathrm{\tilde{s}}}}
\newcommand{\HS}{\ensuremath{\mathrm{\hat{S}}}}
\newcommand{\HM}{\ensuremath{\mathrm{\hat{M}}}}

\newcommand{\tsum}{\textstyle\sum\nolimits}
\newcommand{\mean}{\operatorname{mean}}
\newcommand{\argmax}{\operatorname{arg\,max}}
\newcommand{\argmin}{\operatorname{arg\,min}}
\newenvironment{proof}{\paragraph{Proof:}}{}
\xdefinecolor{Pa1}{RGB}{132, 184, 24} %
\xdefinecolor{Pa2}{RGB}{241, 135, 0} %
\xdefinecolor{Pa3}{RGB}{0, 155, 170} %
\xdefinecolor{Pa4}{RGB}{191, 2, 127} %
\xdefinecolor{Pa5}{RGB}{250, 220, 0} %
\xdefinecolor{Pa6}{RGB}{228, 103, 8} %
\pgfplotsset{compat=newest}

\begin{document}
\let\WriteBookmarks\relax
\def\floatpagepagefraction{1}
\def\textpagefraction{.001}

\shorttitle{Medoid Silhouette clustering with automatic cluster number selection}

\shortauthors{L. Lenssen and E. Schubert}

\title [mode = title]{Medoid Silhouette clustering with automatic cluster number selection}

\tnotemark[1] 

\tnotetext[1]{This is an extended version of Lenssen and Schubert, Clustering by Direct Optimization of the Medoid Silhouette in Similarity Search and Applications. SISAP 2022. \href{https://doi.org/10.1007/978-3-031-17849-8_15}{DOI:10.1007/978-3-031-17849-8\_15}}

%

\author{Lars Lenssen}[orcid=0000-0003-0037-0418]

\ead{lars.lenssen@tu-dortmund.de}

\author{Erich Schubert}[orcid=0000-0001-9143-4880]

\ead{erich.schubert@tu-dortmund.de}

\affiliation{organization={TU Dortmund University},
            addressline={Informatik VIII}, 
            city={Dortmund},
            citysep={}, 
            postcode={44221}, 
            country={Germany}}

\cortext[0]{Corresponding authors} 


\begin{abstract}
The evaluation of clustering results is difficult, highly dependent on the evaluated data set and the perspective of the beholder. There are many different clustering quality measures, which try to provide a general measure to validate clustering results. A very popular measure is the Silhouette. We discuss the efficient medoid-based variant of the Silhouette, perform a theoretical analysis of its properties, provide two fast versions for the direct optimization, and discuss the use to choose the optimal number of clusters. We combine ideas from the original Silhouette with the well-known PAM algorithm and its latest improvements FasterPAM. One of the versions guarantees equal results to the original variant and provides a run speedup of $O(k^2)$. In experiments on real data with 30000 samples and $k$=100, we observed a 10464$\times$ speedup compared to the original PAMMEDSIL algorithm.
Additionally, we provide a variant to choose the optimal number of clusters directly.
\end{abstract}



\begin{NoHyper}
\maketitle 
\end{NoHyper}


\section{Introduction}\label{}
In cluster analysis, the user is interested in discovering previously
unknown structure in the data, as opposed to classification,
where one tries to predict the known structure (i.e., labels) for new data points.
Sometimes, clustering can also be interpreted as data quantization and approximation,
for example, in k-means, where the objective is to minimize the sum of squared errors when approximating
the data with $k$ average vectors, spherical k-means, where we maximize the cosine similarities
to the $k$ centers, and k-medoids, where we minimize the sum of distances
when approximating the data by $k$ data points.
Other clustering approaches such as DBSCAN \cite{DBLP:conf/kdd/EsterKSX96,DBLP:journals/tods/SchubertSEKX17}
cannot easily be interpreted this way, but discover structure related
to connected components and density-based minimal spanning trees \cite{DBLP:conf/lwa/SchubertHM18,Beer/23a}.

The evaluation of clusterings is a challenge, as there are no labels available.
While many internal (``unsupervised'', not relying on external labels)
evaluation measures were proposed such as the Silhouette~\cite{Rousseeuw/87a},
the Davies-Bouldin index~\cite{Davies/79a}, the Variance-Ratio criterion~\cite{Calinski/74a}, the Dunn index~\cite{Dunn/74a}, and many more,
using these indexes for evaluation suffers from inherent challenges.
\citet{DBLP:journals/sadm/VendraminCH10} survey 40 such measures and variants, and find the Silhouette to be one of the most robust.
\citet{Jaskowiak/16a} combine different internal validation measures in an ensemble to improve performance.
Bonner \cite{DBLP:journals/ibmrd/Bonner64} noted that
``none of the many specific definitions [...] seems best in any general sense'',
and results are subjective ``in the eye of the beholder'' as noted by
Estivill-Castro~\cite{DBLP:journals/sigkdd/Estivill-Castro02}.
While these claims refer to clustering methods, not evaluation methods,
we argue that these do not differ substantially:
each internal cluster evaluation method implies a clustering algorithm obtained %
by enumeration of all candidate clusterings, keeping the best.
The main difference between clustering algorithms and internal evaluation then
is whether or not we know an efficient optimization strategy. K-means 
is an optimization strategy for the sum of squares evaluation measure,
while the k-medoids algorithms PAM, and alternating optimization~\cite{Maranzana/63a} are two different strategies for
optimizing the sum of distances from a set of $k$~representatives chosen from the data, a variant of the facility location problem.
In this article, we focus on the evaluation measure known as the
Silhouette~\cite{Rousseeuw/87a}, %
and discuss an efficient algorithm to optimize a variant of this measure,
inspired by the well-known PAM algorithm \cite{Kaufman/Rousseeuw/87a, Kaufman/Rousseeuw/90c}
and FasterPAM~\cite{DBLP:journals/is/SchubertR21,DBLP:conf/sisap/SchubertR19}. Silhouette is also a popular measure to choose the number of clusters in k-medoids or even k-means. However, classic visual inspection of the Silhouette plot is only feasible for small data sets, and users typically to rely on the aggregate coefficient. We can choose the number of clusters performantly even on larger data sets with our more efficient integration.

This article is an extended version of:\\
Lenssen, L., Schubert, E. (2022). Clustering by Direct Optimization of the Medoid Silhouette. In: Similarity Search and Applications. SISAP 2022. LNCS 13590. Springer, Cham. \href{https://doi.org/10.1007/978-3-031-17849-8_15}{DOI:10.1007/978-3-031-17849-8\_15}

In particular, this extended version adds the algorithm for automatically choosing the number of clusters.

\section{Silhouette and Medoid Silhouette}
The Silhouette~\cite{Rousseeuw/87a} is a popular measure to evaluate clustering validity, and
performs very well in empirical studies \cite{DBLP:journals/pr/ArbelaitzGMPP13,DBLP:journals/pr/BrunSHLCSD07}.
For
the given samples $X = \{x_1,\ldots,x_n\}$,
a dissimilarity measure $d:X\times X\rightarrow \mathbb{R}$,
and
the cluster labels $L=\{l_1,\ldots,l_n\}$ for each sample in $X$,
the Silhouette of a single element $i$ is calculated based on the average distance to its own cluster~$a_i$ and the smallest average distance to another cluster~$b_i$ as:
\begin{align*}
s_i(X, d, L) &= \tfrac{b_i-a_i}{\max(a_i,b_i)}
\;\text{, where}\\
a_i &= \phantom{\min\nolimits_{k\neq l_i}\;} \mean \left\{d(x_i, x_j) \mid l_j = l_i, i \neq j\right\}
\\
b_i &= \min\nolimits_{k\neq l_i}\;\mean\left\{d(x_i, x_j) \mid l_j = k\right\}
\;.
\end{align*}
The motivation is that ideally, each point is much closer to the cluster it is assigned to,
than to another ``second closest'' cluster. For $b_i \gg a_i$, the Silhouette approaches 1,
while for points with $a_i = b_i$ we obtain a Silhouette of 0, and negative values can arise
if there is another closer cluster and hence $b_i < a_i$.
If the cluster contains only a single element $a_i$ is undefined in this equation, and \citet{Rousseeuw/87a} uses $s_i=0$ then.
The Silhouette values $s_i$ can then be used to visualize the cluster quality by sorting
objects by label $l_i$ first, and then by descending $s_i$, to obtain the Silhouette plot. Figure~\ref{fig2} shows an example of the visualization for the data set of \citet{Klein/15a}.
\begin{figure}[t]
\begin{subfigure}{0.46\textwidth}\centering
\includegraphics[width=1.0\textwidth]{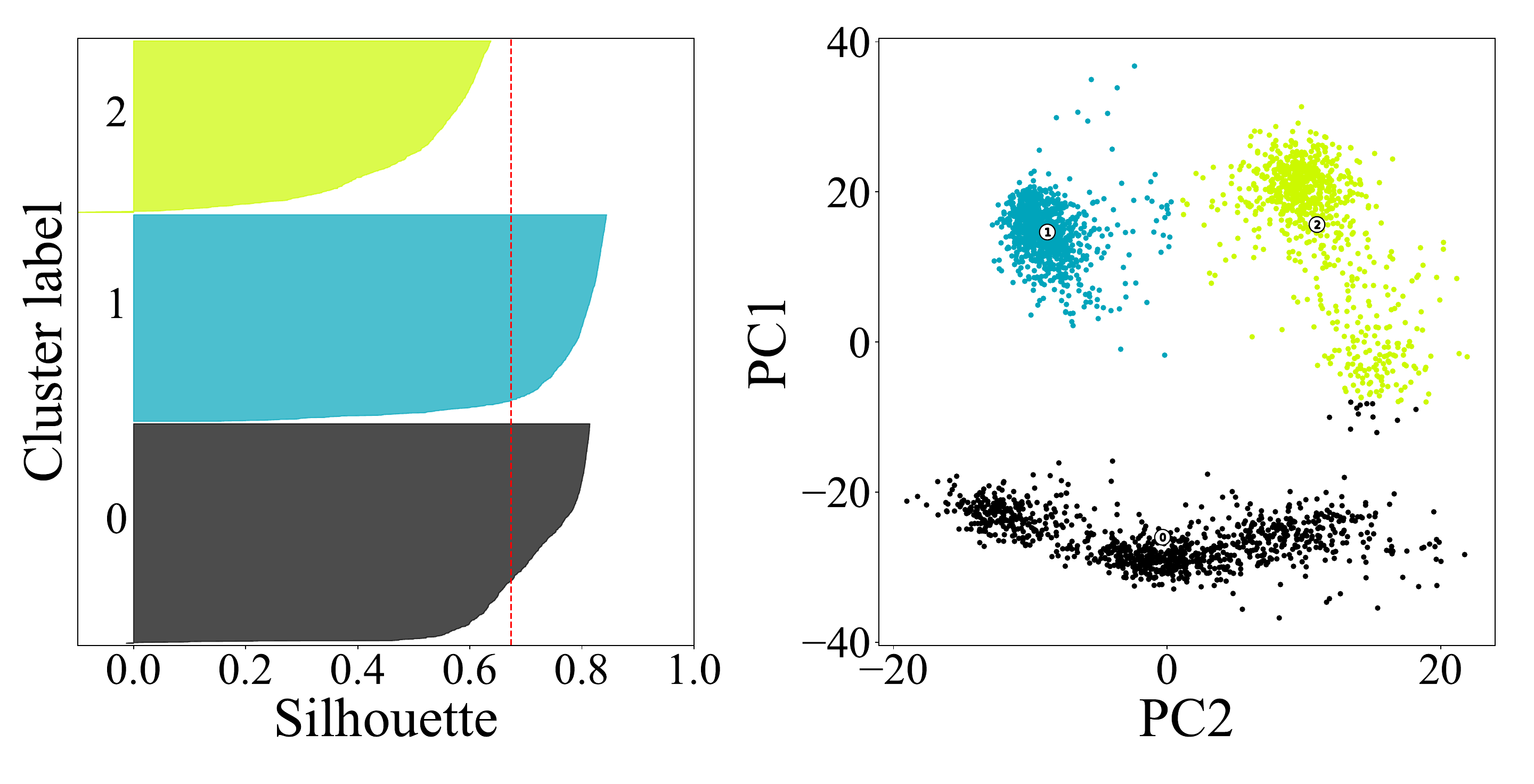}
\caption{$k=3$ with $ASW=0.67$}
\label{fig2.1}
\end{subfigure}
\begin{subfigure}{0.46\textwidth}\centering
\includegraphics[width=1.0\textwidth]{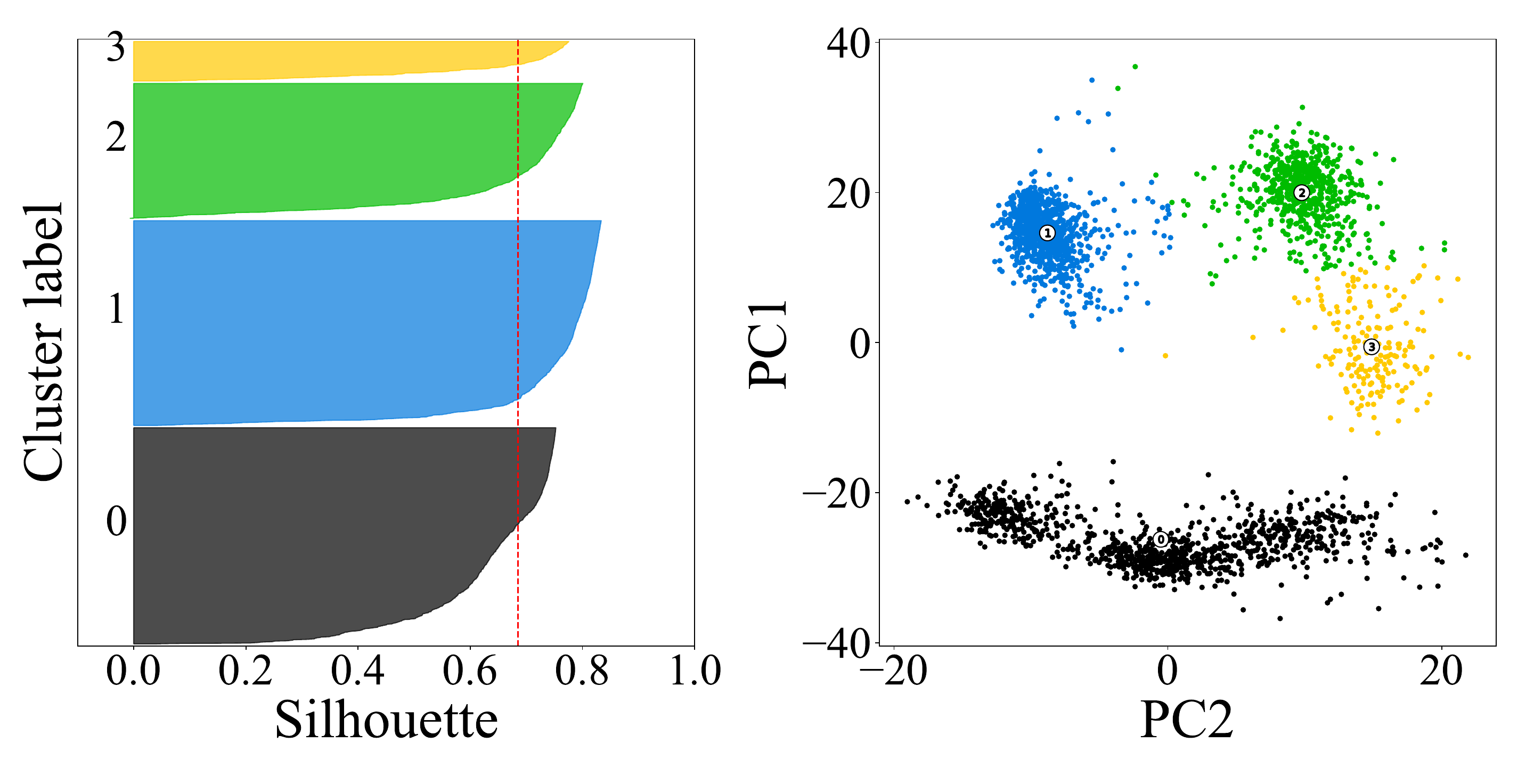}
\caption{$k=4$ with $ASW=0.68$}
\label{fig2.2}
\end{subfigure}
\begin{subfigure}{0.46\textwidth}\centering
\includegraphics[width=1.0\textwidth]{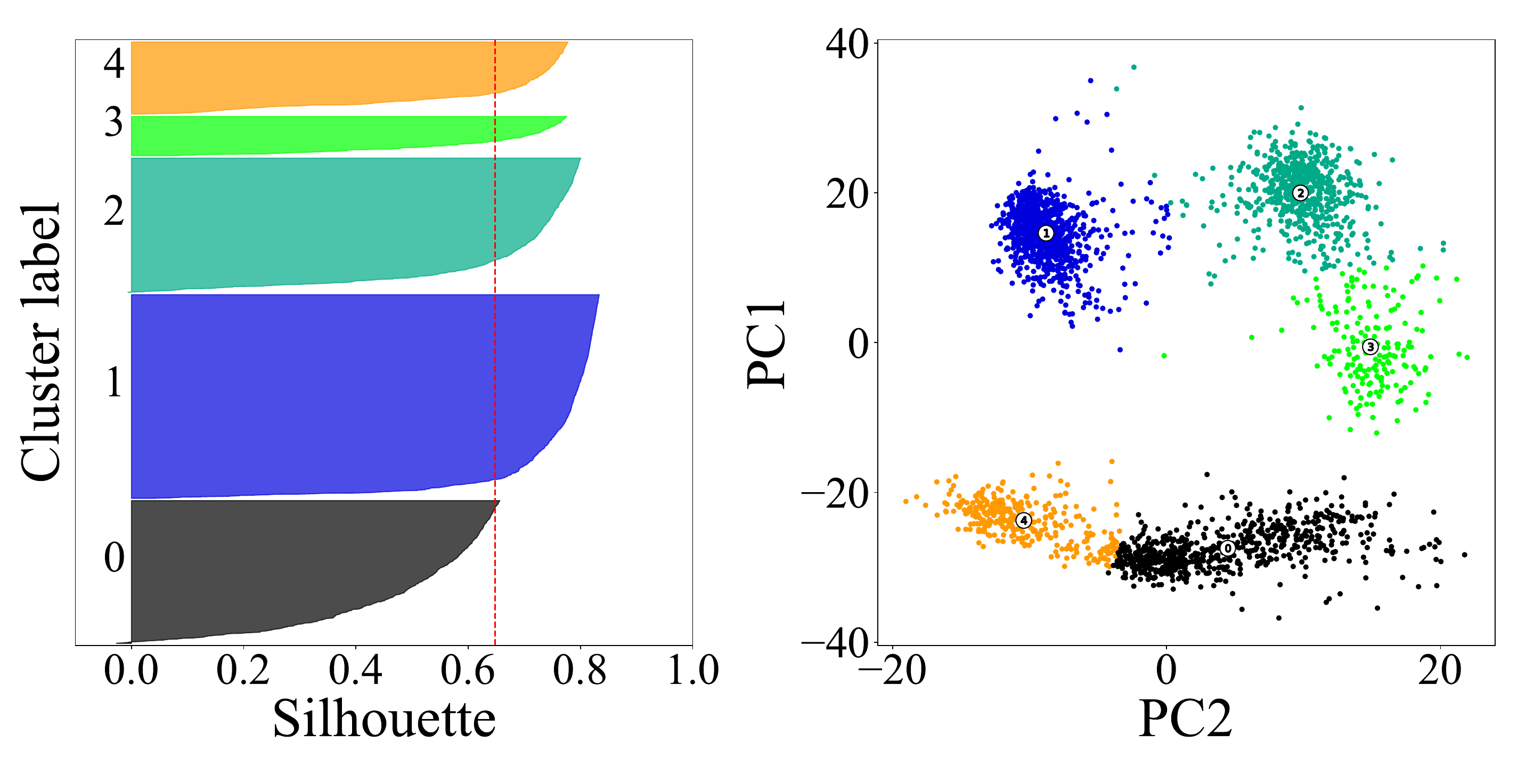}
\caption{$k=5$ with $ASW=0.65$}
\label{fig2.3}
\end{subfigure}
\caption{Average Silhouette Width for clustering results of k-means with $k=3$ to $k=5$ on the data set of \citet{Klein/15a}. Klein’s dataset contains embryonic stem cells measured at \textbf{four} different time points.}
\label{fig2}
\end{figure}
However, visually inspecting the Silhouette plot is only feasible for small data sets,
and hence it is also common to aggregate the values into a single statistic, often referred to
as the Average Silhouette Width (ASW), but also just as ``Silhouette score'' of a clustering:
\begin{align*}
S(X, d, L) = \tfrac{1}{n} \tsum_{i=1}^n s_i(X, d, L)
\;.
\end{align*}
Hence, this is a function that maps a data set, dissimilarity, and cluster labeling to a real number,
and this measure has been shown to satisfy
desirable properties for clustering quality measures (CQM) by \citet{DBLP:conf/nips/Ben-DavidA08}.

A key limitation of the Silhouette is its computational cost. It is easy to see that it requires
all pairwise dissimilarities, and hence takes $O(N^2)$ time to compute -- much more than popular
clustering algorithms such as k-means.

For center-based clustering algorithms such as k-means and k-medoids, a simple approximation to the Silhouette is
possible by using the distance to the cluster centers respectively medoids
$M=\{M_1,\ldots,M_k\}$ instead of the average distance.
For this ``simplified Silhouette'' (which can be computed in $O(N k)$ time,
and which \citet{VanderLaan/03a} called medoid-based Silhouette)
we use the distance to the object cluster center $a_i'$
and the distance to the closest other cluster $b_i'$,
to compute the score $s_i'$ of each sample $i$:
\begin{align*}
s_i'(X, d, M) &= \tfrac{b_i'-a_i'}{\max(a_i',b_i')}
\;\text{, where}\\
a_i' &= \phantom{\min\nolimits_{k\neq l_i}\;{}} d(x_i, M_{l_i})
\\
b_i' &= \min\nolimits_{k\neq l_i}\; d(x_i, M_k)
\;.
\end{align*}
\citet{DBLP:journals/sadm/VendraminCH10} found the simplified Silhouette to perform comparable to the regular Silhouette, and recommend it for large data sets because of the lower computational requirements.

If each point is assigned to the closest cluster center (as in the standard algorithm for k-means, and also optimal for k-medoids and the Silhouette),
we further know that $a_i'\leq b_i'$ and $s_i'\geq 0$,
and hence this can further be simplified to
the \emph{Medoid Silhouette} $\tilde{s}_i$ of sample $i$:
\begin{align*}
\tilde{s}_i(X, d, M) &= \tfrac{d_2(i)-d_1(i)}{d_2(i)}
= 1 - \tfrac{d_1(i)}{d_2(i)}
\;.
\end{align*}
where $d_1$ is the distance to the closest and $d_2$ to the second closest center in~$M$.
For $d_1(i)=d_2(i)=0$, we define $\tilde{s}=1$,
corresponding to adding a negligible small value to $d_2(i)$.
The \emph{Average Medoid Silhouette} (AMS) then is defined as
\begin{align*}
\tilde{S}(X, d, M) = \tfrac{1}{n} \tsum_{i=1}^n \tilde{s}_i(X, d, M)
\;.
\end{align*}
It can easily be seen that the optimum clustering is the (assignment of points to the) optimal set of medoids~$M$
such that we minimize an ``average relative loss``:
\begin{align*}
\argmax_M \tilde{S}(X, d, M) = \argmin_M \mean_i \tfrac{d_1(i)}{d_2(i)}
\end{align*}
i.e., to optimize the relative contrast of the distance to the nearest and the second nearest cluster center.
For clustering around medoids, we impose the restriction $M\subseteq X$;
which has the benefit of not restricting the input data to be numerical,
and allowing non-metric dissimilarity functions~$d$.
This is a key benefit of medoids clustering over, e.g.,
k-means which restricts the input to $X\subseteq\mathbb{R}^d$
in order to be able to compute the cluster means.
On the other hand, this disallows some optimizations particular to the least-squares optimization of k-means.
We argue that for k-means, there exist more meaningful and inexpensive
evaluation measures than the Silhouette,
such as the Variance-Ratio criterion of \citet{Calinski/74a},
as discussed by \citet{Schubert/22b}.

\begin{figure*}[tb]
\begin{subfigure}{.32\linewidth}\centering
\includegraphics[height=.8\linewidth]{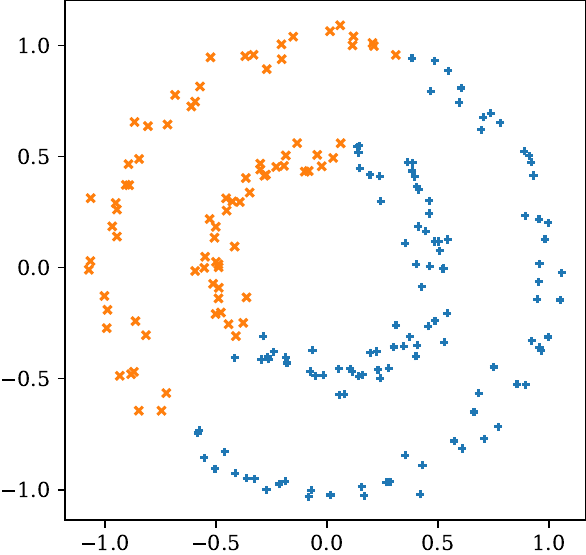}
\caption{Non-convex: nested circles}
\end{subfigure}\hfill%
\begin{subfigure}{.32\linewidth}\centering
\includegraphics[height=.8\linewidth]{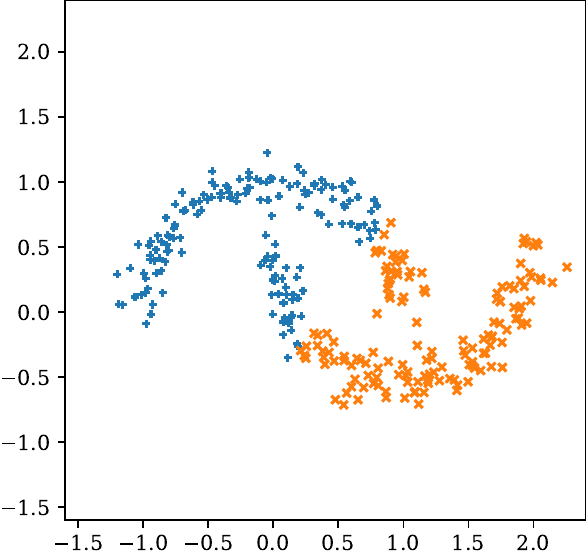}
\caption{Non-convex: two moons}
\end{subfigure}\hfill%
\begin{subfigure}{.32\linewidth}\centering
\includegraphics[height=.8\linewidth]{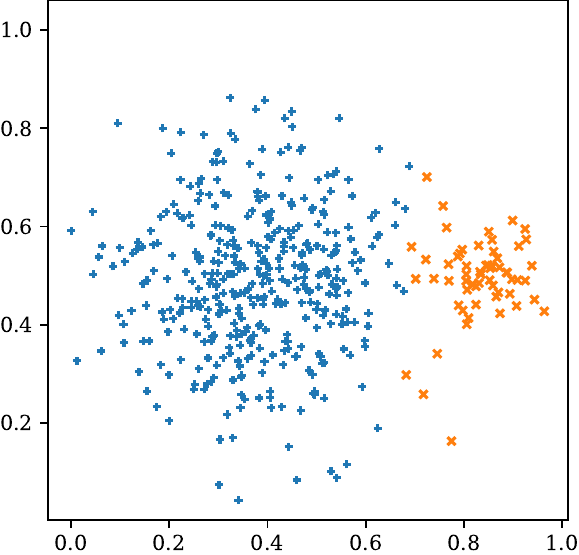}
\caption{Different diameter}
\end{subfigure}
\caption{Toy examples that violate key Silhouette assumptions, and where Silhouette hence prefers undesired solutions.}
\label{fig:silhouette-viloations}
\end{figure*}
From the formulation of both Silhouette and the Medoid Silhouette it should be obvious that the model assumes that each point is best assigned to its ``closest'' cluster, resembling a generalized Voronoi partitioning of the data. Such approaches work best when clusters are spherical and have the same diameter. When clusters are non-convex, but also when clusters have very different size, this assumption may not hold. In such cases, Silhouette may prefer suboptimal solutions.
Figure~\ref{fig:silhouette-viloations} shows some toy data sets that violate these assumptions, and where the optimum Silhouette does not yield the desired solution.
Note that similar limitations exist in many standard clustering algorithms, in particular in k-means and k-medoids. With k-medoids as well as Silhouette clustering it may be possible to use the density connectivity distance of \citet{Beer/23a} that is underlying density-based clustering to improve the results.

\section{Related Work}

The Silhouette~\cite{Rousseeuw/87a} was originally proposed along with
Partitioning Around Medoids (PAM,~\cite{Kaufman/Rousseeuw/87a, Kaufman/Rousseeuw/90c}),
and indeed k-medoids already does a decent job at finding a good solution,
although it optimizes a simpler criterion by minimizing the sum of total deviations.
\citet{VanderLaan/03a} proposed to optimize the Silhouette by substituting
the Silhouette evaluation measure into a simplified variant of the PAM SWAP procedure (calling this PAMSIL).
In each iteration, for each of the $k(N-k)$ possible swaps (exchanging the roles of a medoid and a non-medoid), they compute the full Silhouette in $O(N^2)$,
instead of computing the \emph{change} in the loss used by PAM SWAP.
Algorithm~\ref{alg:pamsil} gives a pseudocode of this procedure.
The complexity of PAMSIL hence increases to $O(k(N-k) N^2)$ per iteration, substantially worse than PAM, which is in $O(k(N-k)^2)$.
Because this yields a very slow clustering method, they also considered the Medoid Silhouette
instead (PAMMEDSIL), which reduced the time complexity to $O(k^2(N-k)N)$ per iteration, still considerably more than the original PAM SWAP iterations.

Schubert and Rousseeuw \cite{DBLP:journals/is/SchubertR21,DBLP:conf/sisap/SchubertR19} recently improved the PAM method,
and their FastPAM approach reduces the cost of PAM by a factor of $O(k)$  by the use of an accumulator array to avoid the innermost loop, making the method $O(N^2)$ per iteration.
In this work, we combine ideas from the FastPAM and the PAMMEDSIL algorithms into an algorithm denoted as FastMSC,
to optimize the Medoid Silhouette with a swap-based local search, but a run time comparable
to FastPAM, i.e., $O(N^2)$ per iteration.
Similar to FasterPAM, we also observe a reduction in the number of iterations at no noticable loss in quality if we eagerly perform swaps, i.e., do a first-descent instead of a steepest-descent optimization in the proposed algorithm FasterMSC.

\begin{algorithm2e}[b]
\caption{PAMSIL SWAP: Optimize Silhouette}
\label{alg:pamsil}
\SetKwBlock{Repeat}{repeat}{}
\SetKw{Break}{break}
$S \gets $ Silhouette sum of the initial solution $M$\;
\Repeat{
    $(S_*, M_*)\gets(0,$null$)$\;
    \ForEach(\tcp*[f]{medoids}){$m_i\in M=\{m_1,\ldots,m_k\}$} {
        \ForEach(\tcp*[f]{non-meds}){$x_j\notin\{m_1,\ldots,m_k\}$} {
            $(S,M') \gets (0, M \setminus \{m_i\} \cup \{x_j\})$\;
            \ForEach(){$x_o\in X=\{x_1,\ldots,x_n\}$} {
                $S \gets S + s_o(X,d,M')$ \tcp*{  Silhouette}
            }
            \lIf%
            {$S > S_*$} {
        		$(S_*, M_*)\gets( S,M')$%
        	}
        }
    }
    \lIf{$S_*\geq S$}{\Break}
    $(S,M) \gets (S_*,M_*)$\tcp*[r]{perform swap}
}
\Return {$(S / N,M)$}\;
\end{algorithm2e}

We will first perform a theoretical analysis of the properties of the Medoid Silhouette, to show that it is worth exploring as an alternative to the original Silhouette, then introduce the new algorithm.

\section{Axiomatic Characterization of Medoid Clustering}
To characterize the Medoid Silhouette,
we follow the axiomatic approach of \citet{DBLP:conf/nips/Ben-DavidA08}, i.e.,
we prove the value of the Average Medoid Silhouette~(AMS) as a clustering quality measure~(CQM) by proving that it satisfies some interesting properties.
\citet{Kleinberg/Jon/02a} defined three axioms for clustering functions
and argued that no clustering algorithm can satisfy these desirable properties at the same time,
as they contradict.
Because of this, \citet{DBLP:conf/nips/Ben-DavidA08}
weaken the original Consistency Axiom and extract four axioms for clustering quality measures:
\emph{Scale Invariance} and \emph{Richness} are defined analogously to the Kleinberg Axioms. We first redefine the CQM axioms~\cite{DBLP:conf/nips/Ben-DavidA08} to match the notion of medoid-based clustering
before we check them for the Average Medoid Silhouette.
\begin{definition}[Relation $\sim_M$]
\label{d1}
For given data points $X = \{x_1,\ldots,x_n\}$ with a set of $k$~medoids $M = \{m_1,\ldots,m_k\}$ and a dissimilarity~$d$, we write $x_i \sim_M x_{i'}$ whenever $x_i$ and~$x_{i'}$ have the same nearest medoid $\nearest(i) \subseteq M$, otherwise $x_i \not\sim_M x_{i'}$.
\end{definition}
The $\sim_M$ relation encodes the partitioning of the data set given by the medoids~$M$, and is transitive and symmetric.
\begin{definition}[M-consistent]
\label{d2}
Dissimilarity $d'$ is an \emph{M-consistent} variant of $d$, if $d'(x_i, x_{i'}) \leq d(x_i, x_{i'})$ for $x_i \sim_M x_{i'}$, and $d'(x_i, x_{i'}) \geq d(x_i, x_{i'})$ for $x_i \not\sim_M x_{i'}$.
\end{definition}
Meaning that a consistent variant reduces the distances within clusters and increases the distances between clusters.

\begin{definition}[Isomorphism of Medoids]
\label{d3}
Two sets of medoids $M, M' \subseteq X$ with a distance function $d$ over~$X$,
are \emph{isomorphic}, if there exists a distance-preserving isomorphism $\phi : X \to X$,
such that for all $x_i, x_{i'} \in X$, $x_i \sim_M x_{i'}$ if and only if $\phi(x_i) \sim_{M'} \phi(x_{i'})$.
\end{definition}

\begin{axiom}[Scale Invariance]
\label{a1}
A medoid-based clustering quality measure~$f$ satisfies \emph{scale invariance} if for every set of medoids $M \subseteq X$ for $d$, and every positive~$\lambda$, $f(X, d, M) = f(X, \lambda d, M)$.
\end{axiom}
Meaning that if we scale all distances with the same constant $\lambda>0$, the quality does not change.

\begin{axiom}[Consistency]
\label{a2}
A medoid-based clustering quality measure~$f$ satisfies \emph{consistency} if for a set of medoids $M \subseteq X$ for $d$, whenever $d'$ is an \mbox{M-consistent} variant of~$d$, then $f(X, d', M) \geq f(X, d, M)$. 
\end{axiom}
Meaning that if we reduce the distances within the same cluster and increase the distances between different clusters, the quality does not decrease.

\begin{axiom}[Richness]
\label{a3}
A medoid-based clustering quality measure~$f$ satisfies \emph{richness} if for each set of medoids $M \subseteq X$, there exists a distance function $d$ over~$X$ such that $M = \argmax_{M'}f(X, d, M')$.
\end{axiom}
Meaning that for every possible set of medoids, there exists a distance for which this solution is optimal.

\begin{axiom}[Isomorphism Invariance]
\label{a4}
A medoid-based clustering quality measure~$f$ is \emph{isomorphism-invariant} if for all
sets of medoids $M, M' \subseteq X$ with distance $d$ over~$X$ where $M$ and $M'$ are isomorphic, $f(X, d, M) = f(X, d, M')$. 
\end{axiom}
Meaning that transformations (including data permutations) which preserve distances do not affect clustering quality.

Batool and Hennig~\cite{Batool/Hennig/21a} prove that the Average Silhouette Width~(ASW)
satisfies the original CQM axioms.
We now prove the four adapted axioms for the Average Medoid Silhouette~(AMS).
\begin{theorem}
The AMS is a \emph{scale invariant} clustering quality measure.
\end{theorem}
\begin{proof}
If we replace $d$ with $\lambda d$, both $d_1(i)$ and $d_2(i)$ are multiplied by $\lambda$, and the term will cancel out. Hence, $\tilde{s}_i$ does not change for any $i$:
\begin{align*}
\tilde{S}(X, \lambda d, M) &= \tfrac{1}{n} \tsum_{i=1}^n \tilde{s}_i(X, \lambda d, M) \\
&= \tfrac{1}{n} \tsum_{i=1}^n 1-\tfrac{\lambda d_1(i)}{\lambda d_2(i)} \\
&= \tfrac{1}{n} \tsum_{i=1}^n 1-\tfrac{d_1(i)}{ d_2(i)} \\
&= \tfrac{1}{n} \tsum_{i=1}^n \tilde{s}_i(X, d, M)
= \tilde{S}(X, d, M)
\;.
\end{align*}
\end{proof}
\begin{theorem}
The AMS is a \emph{consistent} clustering quality measure.
\end{theorem}
\begin{proof}
Let dissimilarity $d'$ be an M-consistent variant of $d$.
By Definition~\ref{d2}: $d'(x_i, x_{i'} ) \leq d(x_i, x_{i'} )$ for all $x_i \sim_{M} x_{i'}$, and
$\min_{x_i \not\sim_{M} x_{i'}}d'(x_i, x_{i'} ) \geq \min_{x_i \not\sim_{M} x_{i'}}d(x_i, x_{i'} )$.
This implies for all $i \in \mathbb{N}$:
$d'_1(i) \leq d_1(i), d'_2(i) \geq d_2(i)$
and it follows:
\begin{align*}
\tfrac{d_1(i)}{ d_2(i)} - \tfrac{ d'_1(i)}{ d'_2(i)} &\geq 0
\quad\Leftrightarrow\quad  %
1-\tfrac{ d'_1(i)}{ d'_2(i)} - \big(1- \tfrac{ d_1(i)}{ d_2(i)}\big) \geq 0
\end{align*}
which is equivalent to $\forall_i \; \tilde{s}_i(X, d', M) \geq \tilde{s}_i(X, d, M)$,
hence $\tilde{S}(X, d', M) \geq \tilde{S}(X, d, M)$, i.e., AMS is a consistent clustering quality measure.
\end{proof}
\begin{theorem}
The AMS is a \emph{rich} clustering quality measure.
\end{theorem}
\begin{proof}
We can simply encode the desired set of medoids $M$ in our dissimilarity~$d$.
We define $d(x_i,x_j)$ such that it is~0 if trivially $i=j$,
or if $x_i$ or $x_j$ is the first medoid $m_1$ and the other is not a medoid itself.
Otherwise, let the distance be~1.

For $M$ we then obtain $\tilde{S}(X, d, M)=1$,
because $d_1(i)=0$ for all objects, as either $x_i$ is a medoid itself, or can
be assigned to the first medoid~$m_1$. This is the maximum possible Average Medoid Silhouette.
Let $M'\neq M$ be any other set of medoids. Then there exists at least one missing
$x_i\in M\setminus M'$. For this object $\tilde{s}_i(X, d, M)=0$ (as its distance to all other objects is 1, and it is not in~$M'$), and hence $\tilde{S}(X, d, M')<1=\tilde{S}(X, d, M)$.
That is, for any set of medoids (or similar, for any partitioning of the data set), there exists a dissimilarity function that yields the desired result.
\end{proof}
\begin{theorem}
The AMS is an \emph{isomorphism-invariant} clustering quality measure.
\end{theorem}
\begin{proof}
Let $M, M' \subseteq X$ be two sets of medoids with a distance function $d$ over~$X$. If they are isomorphic, there exists a distance-preserving isomorphism $\phi : X \to X$,
such that for all $x_i, x_{i'} \in X$, $x_i \sim_M x_{i'}$ if and only if $\phi(x_i) \sim_{M'} \phi(x_{i'})$. 

Therefore, every $d_1(i)$ is equal $d_1(i')$ and every $d_q(i)$ is equal $d_2(i')$, and hence $\tilde{S}(X, d, M) = \tilde{S}(X, d, M')$.
\end{proof}

\FloatBarrier
\section{Direct Optimization of Medoid Silhouette}
PAMSIL~\cite{VanderLaan/03a} is a modification of PAM~\cite{Kaufman/Rousseeuw/87a, Kaufman/Rousseeuw/90c}
to optimize the ASW.
For PAMSIL, \citet{VanderLaan/03a} adjust the SWAP phase of PAM by always performing the SWAP
that provides the best increase in the ASW.
When no further improvement is found, the algorithm terminates and a (local) maximum of the ASW has been achieved.
In contrast to k-means, this procedure does not alternate between two steps, but every step reduces the loss function, or the algorithm stops. Given that there only exist $\choose{k}{n}$ possible states, only a finite number of improvements is possible.
However, where in the original PAM algorithm we efficiently compute only the change in its loss
(in $O(N-k)$ time for each of $(N-k)k$ swap candidates),
PAMSIL computes the entire ASW in $O(N^2)$ for every candidate,
and hence the run time per iteration increases to $O(k(N-k)N^2)$.
For a small $k$, this yields a run time that is cubic in the number of objects $N$,
and the algorithm may need several iterations to converge.
A pseudocode of PAMSIL is given in Algorithm~\ref{alg:pamsil}.

\subsection{Naive Medoid Silhouette Clustering}
PAMMEDSIL~\cite{VanderLaan/03a} uses the Average Medoid Silhouette (AMS) instead,
which can be evaluated in only $O(Nk)$ time. This yields a SWAP run time
of $O(k^2(N-k)N)$ (now only quadratic in~$N$ for small $k\ll N$, but for practical applications the quadratic dependency on $k$ is also noticeable).
A pseudocode of PAMMEDSIL is almost identical to that of PAMSIL shown
in Algorithm~\ref{alg:pamsil}, but using the medoid-based Silhouette scores~$s_o'$ instead of the classic Silhouette~$s_o$.

As Schubert and Rousseeuw \cite{DBLP:journals/is/SchubertR21,DBLP:conf/sisap/SchubertR19} were able to reduce the run time of PAM
to $O(N^2)$ per iteration,
we will now modify the PAMMEDSIL approach accordingly to obtain a similar improvement. We also apply the idea of eager swapping~\cite{DBLP:journals/is/SchubertR21}, i.e., we perform greedy first-descent optimization instead of searching for the steepest-descent.

\subsection{Finding the Best Swap}
\label{sec52}
We first bring PAMMEDSIL up to par with regular PAM.
The trick introduced with PAM is to compute the change in loss instead of recomputing the loss,
which can be done in $O(N-k)$ instead of $O(k(N-k))$ time if we store the distance to the
nearest and second centers, as the latter allows us to compute the change if the current nearest center is removed efficiently.
In the following, we omit the constant parameters $X$ and $d$ for brevity.
We denote the previously nearest medoid of $i$ as $\nearest(i)$,
and $d_1(i)$ is the (cached) distance to it. We similarly define $\second(i)$, $d_2(i)$, and $d_3(i)$
with respect to the second and third nearest medoid. We briefly use $d_1'$ and $d_2'$ to denote
the new distances for a candidate swap.
For the Medoid Silhouette, we can compute the change when swapping medoid
$m_i\in\{m_1,\ldots,m_k\}$ with non-medoid $x_j\notin\{m_1,\ldots,m_k\}$ as follows:
\begin{align*}
\Delta\MS &= \tfrac{1}{n} \tsum_{o=1}^n \Delta\ms_o(M,m_i,x_j)
\\
\Delta\ms_o(M,m_i,x_j) &=  \ms_o(M \setminus \{m_i\} \cup \{x_j\}) - \ms_o(M)
\\
&= 1 - \tfrac{d_1'(i)}{d_2'(i)} - \left(1 - \tfrac{d_1(i)}{d_2(i)}\right)
\\ &=
\tfrac{d_1(i)}{d_2(i)} - \tfrac{d_1'(i)}{d_2'(i)}
\;.
\end{align*}
Clearly, we only need the distances to the closest and second closest center, before \emph{and} after the swap.
Instead of recomputing these values by searching, we exploit that only one medoid can change in a swap.
We can determine the new values of $d_1'$ and $d_2'$
using a constant set of cached values only, and hence save a factor of $O(k)$ on the run time over the naive approach using a loop to check all medoids.

In the PAM algorithm (where the change would be simply $d_1'-d_1$),
the distance to the \emph{second} nearest is cached in order to compute the loss change
if the current medoid is removed, without having to consider all $k-1$ other medoids:
the point is then either assigned to the new medoid, or its former second closest.
To efficiently compute the change in Medoid Silhouette, we have to take this one step further,
and we additionally have to cache the identity of the second closest center (denoted $\second$)
and the distance to the \emph{third} closest center (denoted~$d_3$).
Because we still only change one medoid at a time, both the closest and the second closest must be either the new medoid, or any of the previous three closest (three in case one of them was removed). Depending on which medoid is swapped, we need to distinguish cases.

\begin{algorithm2e}[b]
\caption{Change in Medoid Silhouette, $\Delta\ms_o(M,m_i,x_j)$}
\label{alg:change}
\If%
{$m_i=\nearest(o)$}{
  \tcp{Nearest medoid is replaced with xj:}
  \lIf%
  {$d(o,j)< d_2(o)$}{
    \Return \tabto*{48mm} $\frac{d_1(o)}{d_2(o)}-\frac{d(o,j)}{d_2(o)}$
  }
  \tcp{Nearest removed, xj new second nearest:}
  \lIf%
  {$d(o,j)< d_3(o)$}{
    \Return \tabto*{48mm} $\frac{d_1(o)}{d_2(o)}-\frac{d_2(o)}{d(o,j)}$
  }
  \tcp{Nearest removed, xj is farther than third:}
  \lElse {
    \Return \tabto*{48mm} $\frac{d_1(o)}{d_2(o)}-\frac{d_2(o)}{d_3(o)}$
    \label{alg:change:rem1}
  }
}
\ElseIf%
{$m_i=\second(o)$}{
  \tcp{Second nearest is replaced, xj closer:}
  \lIf%
  {$d(o,j)< d_1(o)$}{
    \Return \tabto*{48mm} $\frac{d_1(o)}{d_2(o)}-\frac{d(o,j)}{d_1(o)}$
  }
  \tcp{Second nearest is replaced by xj:}
  \lIf%
  {$d(o,j)< d_3(o)$}{
    \Return \tabto*{48mm} $\frac{d_1(o)}{d_2(o)}-\frac{d_1(o)}{d(o,j)}$
  }
  \tcp{Second nearest is replaced, xj far:}
  \lElse {
    \Return \tabto*{48mm} $\frac{d_1(o)}{d_2(o)}-\frac{d_1(o)}{d_3(o)}$
    \label{alg:change:rem2}
  }
}
\Else{
  \tcp{xj new closest:}
  \lIf%
  {$d(o,j)< d_1(o)$}{
    \Return \tabto*{48mm} $\frac{d_1(o)}{d_2(o)}-\frac{d(o,j)}{d_1(o)}$
    \label{alg:change:add1}
  }
  \tcp{xj new second closest:}
  \lIf%
  {$d(o,j)< d_2(o)$}{
    \Return \tabto*{48mm} $\frac{d_1(o)}{d_2(o)}-\frac{d_1(o)}{d(o,j)}$
    \label{alg:change:add2}
  }
  \tcp{xj replaced some far medoid:}
  \lElse {
    \Return \tabto*{48mm} 0
  }
}
\end{algorithm2e}

The change in Medoid Silhouette is then computed roughly as follows:
(1) If the new medoid is the new closest, the second closest is either the former nearest, or the second nearest (if the first was replaced).
(2) If the new medoid is the new second closest, the closest either remains the former nearest, or the second
nearest (if the first was replaced).
(3) If the new medoid is neither, we may still have replaced the closest or second closest;
in which case the distance to the third nearest is necessary to compute the new Silhouette.
Putting all the cases (and sub-cases) into one equation becomes a bit messy, and hence we
opt to use the pseudocode in Algorithm~\ref{alg:change} instead of an equivalent mathematical notation.
Note that the first term is always the same (the previous loss), except for the last case,
where it canceled out via $0=\frac{d_1(o)}{d_2(o)}-\frac{d_1(o)}{d_2(o)}$.
As this is a frequent case, it is beneficial to not have further computations here
(and hence, to compute the change instead of computing the loss).
Clearly, this algorithm runs in $O(1)$ if $n_1(o)$, $n_2(o)$, $d_1(o)$, $d_2(o)$, and $d_3(o)$
are known. We also only compute $d(o,j)$ once.
Modifying PAMMEDSIL (Algorithm~\ref{alg:pamsil}) to use an incremental computation yields a run time of $O(k(N-k)N)$ to find the best swap, i.e., already $O(k)$ times faster.
This integrated a key idea of PAM into this algorithm, but we can further improve this approach with ideas from FastPAM.

\subsection{Fast Medoid Silhouette Clustering}
We now integrate an acceleration added to the PAM algorithm by
Schubert and Rousseeuw~\cite{DBLP:conf/sisap/SchubertR19,DBLP:journals/is/SchubertR21},
that exploits redundancy among the loop over the $k$ medoids to replace.
For this, the loss change $\Delta\MS(m_i,x_j)$ is split into multiple components:
(1)~the change by removing medoid $m_i$ (without choosing a replacement, as if we assigned points to their second closest instead),
(2)~the change by adding $x_j$ as an additional medoid (without replacing any of the currently selected), and
(3)~a correction term if both operations occur at the same time.
The first components can be computed in $O(N)$, the second in $O(N(N-k))$, and the last
factor is~0 if the removed medoid is neither of the two closest, and hence is in $O(N)$.
After performing a swap it becomes necessary to update the caches, which may involve finding the third nearest medoid and takes $O(k(N-k))$ time in the worst case.
This then yields an algorithm that finds the best swap in $O((N-k)(N+k))$ -- or less formal in $O(N^2)$, about $k$ times faster for small $k$ than the previous.

First, the changes if we removed medoids $m_i\in M$, and corresponding to lines~\ref{alg:change:rem1}~and~\ref{alg:change:rem2} in Algorithm~\ref{alg:change} are computed for each $i$ as:
\begin{equation}
\Delta\MS^{-m_i} = \sum_{\nearest(o)=i} \tfrac{d_1(o)}{d_2(o)}-\tfrac{d_2(o)}{d_3(o)}
+ \sum_{\second(o)=i} \tfrac{d_1(o)}{d_2(o)}-\tfrac{d_1(o)}{d_3(o)}
\;. \label{eq:removal-mi}
\end{equation}
By iterating over all points $o$ and adding to the accumulators for $\nearest(o)$ and $\second(o)$, we can compute these terms in $O(N)$
for all medoids $m_i$ in one pass.

Secondly, for any non-medoid $x_j$, we can
compute the change when adding this point,
corresponding to the
lines~\ref{alg:change:add1}~and~\ref{alg:change:add2} in Algorithm~\ref{alg:change},
as:
\begin{equation}
\Delta\MS^{+x_j} = \sum_{o=1}^N \begin{cases}
    \frac{d_1(o)}{d_2(o)}-\frac{d(o,j)}{d_1(o)}      & \text{if }d(o,j)< d_1(o) \\ 
    \frac{d_1(o)}{d_2(o)}-\frac{d_1(o)}{d(o,j)}     & \text{elif }d(o,j)< d_2(o) \\
    0     & \text{otherwise} 
    \end{cases}
\;. \label{eq:add-xj}
\end{equation}

Using these two terms, we can derive the remaining
correction term by only considering the cases where $\nearest(o)=i$ or $\second(o)=i$, including
cancel-out terms for summands in $\Delta\MS^{-m_i}$ and $\Delta\MS^{+x_j}$ where necessary:
\begin{align*}
    &\Delta\MS(m_i,x_j) =
    \sum_{o=1}^N\Delta\ms_o(M,m_i,x_j)
    \\
    &=\Delta\MS^{-m_i} + \Delta\MS^{+x_j} \notag\\ 
    &+\! \sum_{\nearest(o)=i} \begin{cases}
    \frac{d(o,j)}{d_1(o)}{+}\frac{d_2(o)}{d_3(o)}{-}\frac{d_1(o)+d(o,j)}{d_2(o)} &\! \text{if }d(o,j)< d_1(o)\\ 
    \frac{d_1(o)}{d(o,j)}{+}\frac{d_2(o)}{d_3(o)}{-}\frac{d_1(o)+d(o,j)}{d_2(o)} &\! \text{elif }d(o,j)< d_2(o) \\
    \frac{d_2(o)}{d_3(o)}-\frac{d_2(o)}{d(o,j)}     &\! \text{elif }d(o,j)< d_3(o) \\
    0     &\! \text{otherwise} 
    \end{cases} \notag\\
    &+\! \sum_{\second(o)=i} \begin{cases}
    \frac{d_1(o)}{d_3(o)}-\frac{d_1(o)}{d_2(o)}     & \text{if }d(o,j)< d_1(o) \\ 
    \frac{d_1(o)}{d_3(o)}-\frac{d_1(o)}{d_2(o)}     & \text{elif }d(o,j)< d_2(o) \\
    \frac{d_1(o)}{d_3(o)}-\frac{d_1(o)}{d(o,j)}     & \text{elif }d(o,j)< d_3(o) \\
    0     & \text{otherwise} 
    \end{cases}
\;.
\end{align*}

\begin{algorithm2e*}[bp]
\caption{FastMSC: Optimizing the Medoid Silhouette}
\label{alg:fastmsc}
\SetKwBlock{Repeat}{repeat}{}
\SetKw{BreakOuterLoopIf}{break outer loop if}
\Repeat{
\lForEach(\label{alg:fastmsc-loop1}){$x_o$}{update $\nearest(o), \second(o), d_1(o), d_2(o), d_3(o)$} 
    $\Delta\MS^{-m_1},\ldots,\Delta\MS^{-m_i} \gets$ compute loss change removing $m_i$ using \eqref{eq:removal-mi}\;
    $(\Delta\MS^*, m^*, x^*)\gets(0,$null$,$null$)$\;
    \ForEach(\tcp*[f]{each non-medoid}\label{alg:fastmsc-loop2}){$x_j\notin\{m_1,\ldots,m_k\}$} {
        $\Delta\MS_i,\ldots,\Delta\MS_k\gets(\Delta\MS^{-m_1},\ldots,\Delta\MS^{-m_i})$\label{alg:fastmscl6}\tcp*[r]{use removal loss}
        $\Delta\MS^{+x_j}\gets0$\label{alg:fastmscl7}\tcp*[r]{initialize shared accumulator}
        \ForEach(\label{alg:fastpms-loop3}){$x_o\in\{x_1,\ldots,x_n\}$} {
            $d_{oj}\gets d(x_o,x_j)$\tcp*[r]{distance to new medoid}
            \If(\tcp*[f]{new closest}\label{alg:fastmscl10}) {$d_{oj} < d_1(o)$} {
            $\Delta\MS^{+x_j}\gets\Delta\MS^{+x_j}+d_1(o)/d_2(o)-d_{oj}/d_1(o)$\label{alg:fastmscl11}\;
        	$\Delta\MS_{\nearest(o)}\gets \Delta\MS_{\nearest(o)}+ d_{oj}/d_1(o) + d_2(o)/d_3(o) - (d_1(o)+d_{oj})/d_2(o)$\label{alg:fastmscl12}\;
        	$\Delta\MS_{\second(o)}\gets \Delta\MS_{\second(o)}+d_1(o)/d_3(o) - d_1(o)/d_2(o)$\;
            }
            \ElseIf(\tcp*[f]{new first/second closest}\label{alg:fastmsc-if2}) {$d_{oj} < d_2(o)$} {
            $\Delta\MS^{+x_j}\gets\Delta\MS^{+x_j}+d_1(o)/d_2(o)-d_1(o)/d_{oj}$\label{alg:fastmscl15}\;
        	$\Delta\MS_{\nearest(o)}\gets \Delta\MS_{\nearest(o)}+d_1(o)/d_{oj} + d_2(o)/d_3(o) - (d_1(o)+d_{oj})/d_2(o)$\label{alg:fastmscl16}\;
        	$\Delta\MS_{\second(o)}\gets \Delta\MS_{\second(o)}+d_1(o)/d_3(o) - d_1(o)/d_2(o)$\;
            }
            \ElseIf(\tcp*[f]{new second/third closest}\label{alg:fastmsc-if3}) {$d_{oj} < d_3(o)$} {
        	$\Delta\MS_{\nearest(o)}\gets \Delta\MS_{\nearest(o)}+d_2(o)/d_3(o) - d_2(o)/d_{oj}$\;
        	$\Delta\MS_{\second(o)}\gets \Delta\MS_{\second(o)}+d_1(o)/d_3(o) - d_1(o)/d_{oj}$\label{alg:fastmscl20}\;
            }
        }
        $i\gets \argmax\Delta\MS_i$\;
        $\Delta\MS_i\gets\Delta\MS_i+\Delta\MS^{+x_j}$\;
        \lIf(\label{alg1-if1}) {$\Delta\MS_i > \Delta\MS^*$} {
        	$(\Delta\MS^*, m^*, x^*)\gets(\Delta\MS,m_i,x_j)$
        }
    }
    \BreakOuterLoopIf{$\Delta\MS^*\leq0$}\;
    swap roles of medoid $m^*$ and non-medoid $x^*$\tcp*[r]{perform swap}
    $\MS\gets\MS+\Delta\MS^*$\;
}
\Return{$\MS,M$}\;
\end{algorithm2e*}
Computing these additional summands takes $O(N)$ time by iterating over all objects $x_o$,
and adding their contributions to accumulators for $n_1(o)$ and $n_2(o)$.

Once we have identified the best swap, we apply this change and enter the next iteration. In Line~\ref{alg:fastmsc-loop1}, the update can also be optimized to use the cached values and only scan for the third closest if one of the closest was removed.

This then gives Algorithm~\ref{alg:fastmsc}, which computes
$\Delta\MS^{+x_j}$ along with the sum of $\Delta\MS^{-m_i}$ and
these correction terms in an accumulator array.
To help with reimplementing our approach efficiently, we give the final simplified equations in the pseudocode, as the intuition has already been explained above.
The algorithm needs $O(k)$ memory for the accumulators in the loop,
and $O(N)$ additional memory to store the cached $n_1$, $n_2$, $d_1$, $d_2$, and $d_3$ for each object.

This algorithm gives the same result,
but FastMSC (``Fast Medoid Silhouette Clustering'') is $O(k^2)$ faster than the naive PAMMEDSIL, as evidenced by the two main nested loops of FastMSC having $(N-k)\times N$ executions containing only $O(1)$ operations inside.

\subsection{Eager Swapping and Random Initialization}

We can now integrate further improvements by
\citet{DBLP:journals/is/SchubertR21}.
Because doing the best swap (steepest descent) does not appear to commonly find better solutions,
but requires a pass over the entire data set for each step,
we can converge to local optima much faster if we perform every swap that yields an improvement,
even though this means we may repeatedly replace the same medoid and perform ``unnecessary'' swaps, because the cost for searching is significantly higher than for performing a swap. For PAM this was called eager swapping, and yields the variant FasterPAM.
This does not improve theoretical run time (the last iteration will always require a pass over the
entire data set to detect convergence), but empirically reduces the number of iterations substantially, while increasing the number of swaps only slightly.
It will no longer find the same results, but there is no evidence that a steepest descent is beneficial
over choosing the first descent found.
\begin{algorithm2e*}[tbp]
\caption{FasterMSC: FastMSC with eager swapping}
\label{alg:fastermsc}
\SetKwBlock{Repeat}{repeat}{}
\SetKw{BreakOuterLoopIf}{break outer loop if}
$x_{\text{last}}\gets$invalid\;
\lForEach(\label{alg:fastermsc-loop1}){$x_o$}{update $\nearest(o), \second(o), d_1(o), d_2(o), d_3(o)$} 
    $\Delta\MS^{-m_1},\ldots,\Delta\MS^{-m_i} \gets$ compute loss change removing $m_i$ using \eqref{eq:removal-mi}\;
\Repeat{
    \ForEach(\tcp*[f]{each non-medoid}\label{alg:fastermsc-loop2}){$x_j\notin\{m_1,\ldots,m_k\}$} {
        \BreakOuterLoopIf{$x_j = x_{\text{last}}$}\;
        $\Delta\MS\gets(\Delta\MS^{-m_1},\ldots,\Delta\MS^{-m_i})$\label{alg:fastermsc1}\tcp*[r]{use removal loss}
        $\Delta\MS^{+x_j}\gets0$\label{alg:fastermsc2}\tcp*[r]{initialize shared accumulator}
        \ForEach(\label{alg:fastermsc-loop3}){$x_o\in\{x_1,\ldots,x_n\}$} {
            $d_{oj}\gets d(x_o,x_j)$\tcp*[r]{distance to new medoid}
            \If(\tcp*[f]{new closest}\label{alg:fastermsc-if1}) {$d_{oj} < d_1(o)$} {
            $\Delta\MS^{+x_j}\gets\Delta\MS^{+x_j}+d_1(o)/d_2(o)-d_{oj}/d_1(o)$\label{alg:fastermsc4}\;
        	$\Delta\MS_{\nearest(o)}\gets \Delta\MS_{\nearest(o)}+ d_{oj}/d_1(o) + d_2(o)/d_3(o) - (d_1(o)+d_{oj})/d_2(o)$\label{alg:fastermsc5}\;
        	$\Delta\MS_{\second(o)}\gets \Delta\MS_{\second(o)}+d_1(o)/d_3(o) - d_1(o)/d_2(o)$\;
            }
            \ElseIf(\tcp*[f]{new first/second closest}\label{alg:fastermsc-if2}) {$d_{oj} < d_2(o)$} {
            $\Delta\MS^{+x_j}\gets\Delta\MS^{+x_j}+d_1(o)/d_2(o)-d_1(o)/d_{oj}$\label{alg:fastermsc6}\;
        	$\Delta\MS_{\nearest(o)}\gets \Delta\MS_{\nearest(o)}+d_1(o)/d_{oj} + d_2(o)/d_3(o) - (d_1(o)+d_{oj})/d_2(o)$\label{alg:fastermsc7}\;
        	$\Delta\MS_{\second(o)}\gets \Delta\MS_{\second(o)}+d_1(o)/d_3(o) - d_1(o)/d_2(o)$\;
            }
            \ElseIf(\tcp*[f]{new second/third closest}\label{alg:fastermsc-if3}) {$d_{oj} < d_3(o)$} {
        	$\Delta\MS_{\nearest(o)}\gets \Delta\MS_{\nearest(o)}+d_2(o)/d_3(o) - d_2(o)/d_{oj}$\;
        	$\Delta\MS_{\second(o)}\gets \Delta\MS_{\second(o)}+d_1(o)/d_3(o) - d_1(o)/d_{oj}$\label{alg:fastermsc8}\;
            }
        }
        $i\gets \argmax\Delta\MS_i$\tcp*[r]{choose best medoid}
        $\Delta\MS_i\gets\Delta\MS_i+\Delta\MS^{+x_j}$\tcp*[r]{add accumulator}
        \If(\tcp*[f]{eager swapping}\label{alg:fastermsc-if4}) {$\Delta\MS_i < 0$} {
        	swap roles of medoid $m_i$ and non-medoid $x_o$\tcp*[r]{perform swap}
                $\MS\gets\MS+\Delta\MS^*$\;
                update $\Delta\MS^{-m_1},\ldots,\Delta\MS^{-m_i}$\;
                $x_{\text{last}}\gets x_o$\;
        }
    }
}
\Return{$\MS,M$}\;
\end{algorithm2e*}
The main downside to this is, that it increases the dependency on the data ordering, and hence is best
used on shuffled data when run repeatedly.
Similarly, we will study a variant that eagerly performs the first swap that improves the AMS
as FasterMSC (``Fast and Eager Medoid Silhouette Clustering'').

Similar to \citet{DBLP:journals/is/SchubertR21}, where the PAM BUILD initialization had become a bottleneck, we also choose a random initialization.
A single pass over the data set with eager swapping tends to find better solutions than the best initialization strategies, and uniform sampling of medoids is very cheap.

\subsection{Choosing the Number of Clusters}
There are many different approaches to determine the optimal number of clusters, yet this remains a challenging task.
For k-means, the so-called Elbow method is commonly called, but is not very well suited for this purpose, and some alternatives have been surveyed and evaluated by \citet{Schubert/22b}.
Because the primary objective of k-means (i.e., the sum of squared deviations from the nearest mean) as well as that of k-medoids (the sum of deviations from the nearest medoid) improves as we increase the number of clusters, it is common to rely on a secondary quality criterion, i.e., an evaluation measure, that does not have this property.
One popular such measure is indeed the Silhouette (e.g., \cite{DBLP:journals/pr/ArbelaitzGMPP13,DBLP:journals/pr/BrunSHLCSD07}), which tends to drop once clusters get too close to each other, unless they are well separated.

\begin{algorithm2e}[t!]
\caption{DynMSC: FasterMSC for dynamic k}
\label{alg:dynmsc}
\SetKwBlock{Repeat}{repeat}{}
\SetKw{BreakOuterLoopIf}{break outer loop if}
$k \gets \text{max}k$ \;
\While{$k \geq 2$}{
    $x_{\text{last}}\gets$invalid\;
        \ForEach(\label{alg:dynmsc-loop1}){$x_o$}{update $\nearest(o), \second(o), d_1(o), d_2(o), d_3(o)$} 
        $\Delta\MS^{-m_1},\ldots,\Delta\MS^{-m_i} \gets$ compute loss change \eqref{eq:removal-mi}\;
    \Repeat{
        \ForEach(\label{alg:dynmsc1-loop2}){$x_j\notin\{m_1,\ldots,m_k\}$} {
            \BreakOuterLoopIf{$x_j = x_{\text{last}}$}\;
            $\Delta\MS\gets(\Delta\MS^{-m_1},\ldots,\Delta\MS^{-m_i})$\label{alg:dynmsc1}\;
            $\Delta\MS^{+x_j}\gets0$\label{alg:dynmsc2}\;
            \ForEach(\label{alg:dynmsc-loop3}){$x_o\in\{x_1,\ldots,x_n\}$} {
            \vdots\tcp*[r]{see FasterMSC}\label{alg:dynmsc2-omit}
            }
         $i\gets \argmax\Delta\MS_i$\;
        $\Delta\MS_i\gets\Delta\MS_i+\Delta\MS^{+x_j}$\;
        \If(\label{alg:dynmsc-if4}) {$\Delta\MS_i < 0$} {
        	swap roles of $m_i$ and $x_o$\;
                $\MS\gets\MS+\Delta\MS^*$\;
                update $\Delta\MS^{-m_1},\ldots,\Delta\MS^{-m_i}$\;
                $x_{\text{last}}\gets x_o$\;
            }
    }
    }
    $\HS_k\gets\MS$\;
    $\HM_k\gets M$\; 
    $i \gets \argmax\Delta\MS$\tcp*[r]{choose medoid to remove}
    remove medoid $m_i$\tcp*[r]{choose k with highest AMS}
    $k \gets k - 1$\;
}
$j\gets \argmax\HS$\;
\Return{$\HS_j,\HM_j$}
\end{algorithm2e}

A repeated computation of the Silhouette is expensive on larger data sets.
Instead of repeatedly running FastMSC with a different number of clusters and keeping the best result, we present a variant of FastMSC that does not require the number of clusters to be given.
In DynMSC, we begin with a maximum number of clusters, optimize the Average Medoid Silhouette,
then decrease the number of clusters by one,
and repeat until we have reached a minimum number of clusters.
During this process, we store the solution with the highest AMS to return later.
However, we integrate this directly with FasterMSC to save redundant computations.
At the end of the FasterMSC optimization, we already know the removal loss of each medoid (c.f., Eq.~\ref{eq:removal-mi}), and instead of removing a random medoid, we can remove the one that incurs the least reduction in AMS, which gives us better starting conditions. We can also retain some of the cached data, and only need to find the nearest medoids for those, where one of the three nearest medoids has been removed. For others, this metadata remains valid. Removing a medoid is very similar to performing a swap, except that there is no replacement medoid. We then can continue with the reduced~$k$ instead of restarting from scratch.
Algorithm~\ref{alg:dynmsc} gives a pseudocode for DynMSC. In Line~\ref{alg:dynmsc2-omit} we omitted code of the inner loop that is identical to FasterMSC.

\section{Experiments}
We next evaluate clustering quality, to show the benefits of optimizing AMS.
We report both AMS and ASW, as well as the supervised measures Adjusted Random Index (ARI) and
Normalized Mutual Information (NMI) that require labeled data.
Afterward, we study the scalability to verify the empirical speedup for our algorithms FastMSC, FasterMSC, and DynMSC.

\begin{figure}[tbp]
\begin{subfigure}{0.46\textwidth}\centering
\begin{tikzpicture}[font=\rmfamily\small]
\begin{axis}[unit vector ratio*=1 1 1, width=1.1\textwidth, xmin = -100, xmax = 150, ymin = -100, ymax = 150, xlabel={PC1}, ylabel={PC2}, xticklabels={,,}, yticklabels={,,}]
\addplot +[only marks, mark=text, text mark=1, mark options={scale=0.6}, Pa1] table [x=b, y=c, col sep=comma] {results/kolodziejczyk/counttable_al_2i.csv};\label{plot1_l_1}
\addplot +[only marks, mark=text, text mark=2, mark options={scale=0.6}, Pa2] table [x=b, y=c, col sep=comma] {results/kolodziejczyk/counttable_al_a2i.csv};\label{plot1_l_2}
\addplot +[only marks, mark=text, text mark=3, mark options={scale=0.6}, Pa3] table [x=b, y=c, col sep=comma] {results/kolodziejczyk/counttable_al_lif.csv};\label{plot1_l_3}
\addlegendimage{/pgfplots/refstyle=plot1_l_1}\addlegendentry{2i}
\addlegendimage{/pgfplots/refstyle=plot1_l_2}\addlegendentry{a2i}
\addlegendimage{/pgfplots/refstyle=plot1_l_3}\addlegendentry{lif}
\end{axis}
\end{tikzpicture}
\caption{\citet{Kolodziejczyk/15a}}
\label{plot2_1}
\end{subfigure}
\begin{subfigure}{0.46\textwidth}\centering
\begin{tikzpicture}[font=\rmfamily\small]
\begin{axis}[unit vector ratio*=1 1 1, width=1.1\textwidth, xmin = -30, xmax = 70, ymin = -50, ymax = 50, xlabel={PC2}, ylabel={PC1}, xticklabels={,,}, yticklabels={,,}]
\addplot +[only marks, mark=text, text mark=1, mark options={scale=0.6}, Pa1] table [x=c, y=b, col sep=comma] {results/klein/counttable_klein_lif.csv};\label{plot2_l_1}
\addplot +[only marks, mark=text, text mark=2, mark options={scale=0.6}, Pa2] table [x=c, y=b, col sep=comma] {results/klein/counttable_klein_2d.csv};\label{plot2_l_2}
\addplot +[only marks, mark=text, text mark=4, mark options={scale=0.6}, Pa3] table [x=c, y=b, col sep=comma] {results/klein/counttable_klein_4d.csv};\label{plot2_l_3}
\addplot +[only marks, mark=text, text mark=7, mark options={scale=0.6}, Pa4] table [x=c, y=b, col sep=comma] {results/klein/counttable_klein_7d.csv};\label{plot2_l_4}
\addlegendimage{/pgfplots/refstyle=plot2_l_1}\addlegendentry{lif}
\addlegendimage{/pgfplots/refstyle=plot2_l_2}\addlegendentry{+2d}
\addlegendimage{/pgfplots/refstyle=plot2_l_3}\addlegendentry{+4d}
\addlegendimage{/pgfplots/refstyle=plot2_l_4}\addlegendentry{+7d}
\end{axis}
\end{tikzpicture}
\caption{\citet{Klein/15a}}
\label{plot2_2}
\end{subfigure}
\caption{Different kind of mouse embryonic stem cells (mESCs). For both data sets we have done PCA and plot the first two principal components. (a) shows 704 mESCs grown in three different conditions and (b) 2717 mESCs at the moment of LIF withdrawal, 2 days after, 4 days after, and 7 days after.}
\label{plot2}
\end{figure}
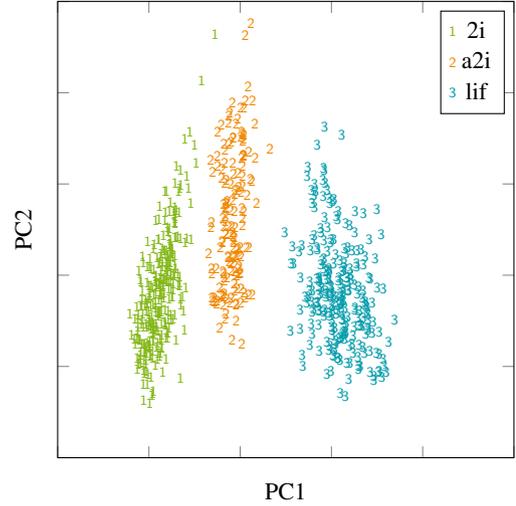
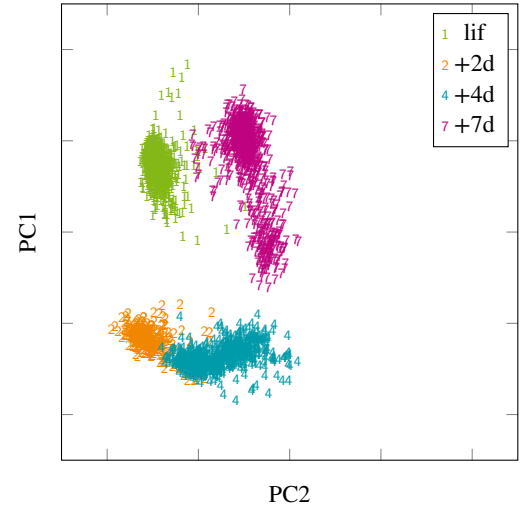

\subsection{Data Sets}
Since it became possible to map gene expression at the single-cell level by RNA sequencing, clustering on these has become a popular task, and Silhouette is a popular evaluation measure there. \citet{Hie/19a} use Silhouette coefficient distributions to compare 26 diverse scRNA-seq datasets under different parameters.
Single-cell RNA sequencing (scRNA-seq) provides high-dimensional data that requires appropriate preprocessing to extract information.
After extraction of significant genes, these marker genes are validated by clustering of proper cells.
We explore two publicly available data sets with larger sample size (by scRNA standards; the data size is not challenging for cluster analysis) of scRNA-sequencing of mouse embryonic stem cells~(mESCs).
\citet{Kolodziejczyk/15a} studied 704 mESCs with 38561 genes grown in three different conditions (2i, a2i, and serum).
\citet{Klein/15a} worked on the influence leukemia inhibitory factor (LIF) withdrawal on mESCs.
For this, they studied a total of 2717 mESCs with 24175 genes.
The data included 933 cells after LIF-withdrawal, 303 cells two days after, 683~cells 4~days after, and 798~cells 7~days after.
We normalize each cell by the total counts over all genes, so that every cell has a total count equal to the median of total counts for observations (cells) before normalization, then we perform principal component analysis (PCA) and use the first three principal components for clustering. This preprocessing matches the procedure of \citet{Kolodziejczyk/15a} in clustering the mESCs Grown in serum, 2i, and a2i media from the original publication.

To test the scalability of our new variants, we need larger data sets.
We use the well-known MNIST data set, with 784~features and 60000~samples
(PAMSIL will not be able to handle this size in reasonable time).
We implemented our algorithms in Rust, extending the \texttt{kmedoids} package~\cite{Schubert/22a},
wrapped with Python, and we make our source code available in this package.
We perform all computations in the same package,
to avoid confounding factors caused by comparing two different implementations~\cite{Kriegel/Schubert/17a}.
We run 10 restarts on an AMD EPYC 7302 processor using a single thread, and evaluate the average values.

\subsection{Clustering Quality}
We evaluated all methods with PAM BUILD initialization and uniform random initialization.
To evaluate the relevancy of the Average Silhouette Width and the Average Medoid Silhouette,
we compare to the true labels using the Adjusted Rand Index (ARI) and Normalized Mutual Information (NMI), two common external measures in clustering.
\begin{figure}[tbp]
\centering
\begin{subfigure}{0.46\textwidth}
\begin{tikzpicture}[font=\rmfamily\small]
\begin{axis}[unit vector ratio*=1 1 1, width=1.1\textwidth, xmin = -100, xmax = 150, ymin = -100, ymax = 150, xlabel={PC1}, ylabel={PC2}, xticklabels={,,}, yticklabels={,,}]
\addplot +[mark options={scale=0.6}, only marks, mark=text, text mark=1, fill=black, color = black] table [x=b, y=c, col sep=comma] {results/kolodziejczyk/zpammedsil_buildlabels_false.csv};
\addplot +[mark options={scale=0.6}, only marks, mark=text, text mark=1,Pa1] table [x=b, y=c, col sep=comma] {results/kolodziejczyk/zpammedsil_buildlabels_true.csv};
\addplot +[mark options={scale=0.6}, only marks, mark=text, text mark=2, fill=black, color = black] table [x=b, y=c, col sep=comma] {results/kolodziejczyk/zpammedsil_buildlabels1_false.csv};
\addplot +[mark options={scale=0.6}, only marks, mark=text, text mark=2, Pa2] table [x=b, y=c, col sep=comma] {results/kolodziejczyk/zpammedsil_buildlabels1_true.csv};
\addplot +[mark options={scale=0.6}, only marks, mark=text, text mark=3, fill=black, color = black] table [x=b, y=c, col sep=comma] {results/kolodziejczyk/zpammedsil_buildlabels2_false.csv};
\addplot +[mark options={scale=0.6}, only marks, mark=text, text mark=3, Pa3] table [x=b, y=c, col sep=comma] {results/kolodziejczyk/zpammedsil_buildlabels2_true.csv};
\end{axis}
\end{tikzpicture}
\caption{Results for PAMMEDSIL (BUILD)}
\label{plot7_1}
\end{subfigure}
\begin{subfigure}{0.46\textwidth}
\begin{tikzpicture}[font=\rmfamily\small]
\begin{axis}[unit vector ratio*=1 1 1, width=1.1\textwidth, xmin = -100, xmax = 150, ymin = -100, ymax = 150, xlabel={PC1}, ylabel={PC2}, xticklabels={,,}, yticklabels={,,}]
\addplot +[mark options={scale=0.6}, only marks, mark=text, text mark=1, fill=black, color = black] table [x=b, y=c, col sep=comma] {results/kolodziejczyk/zpamsil_buildlabels_false.csv};
\addplot +[mark options={scale=0.6}, only marks, mark=text, text mark=1,Pa1] table [x=b, y=c, col sep=comma] {results/kolodziejczyk/zpamsil_buildlabels_true.csv};
\addplot +[mark options={scale=0.6}, only marks, mark=text, text mark=2, fill=black, color = black] table [x=b, y=c, col sep=comma] {results/kolodziejczyk/zpamsil_buildlabels1_false.csv};
\addplot +[mark options={scale=0.6}, only marks, mark=text, text mark=2,Pa2] table [x=b, y=c, col sep=comma] {results/kolodziejczyk/zpamsil_buildlabels1_true.csv};
\addplot +[mark options={scale=0.6}, only marks, mark=text, text mark=3, fill=black, color = black] table [x=b, y=c, col sep=comma] {results/kolodziejczyk/zpamsil_buildlabels2_false.csv};
\addplot +[mark options={scale=0.6}, only marks, mark=text, text mark=3,Pa3] table [x=b, y=c, col sep=comma] {results/kolodziejczyk/zpamsil_buildlabels2_true.csv};
\end{axis}
\end{tikzpicture}
\caption{Results for PAMSIL (BUILD)}
\label{fig52}
\end{subfigure}
\caption{Clustering results for the scRNA-seq data sets of \citet{Kolodziejczyk/15a} for PAMMEDSIL and PAMSIL. All correctly predicted labels are colored by the corresponding cluster and all errors are marked as black. }
\end{figure}
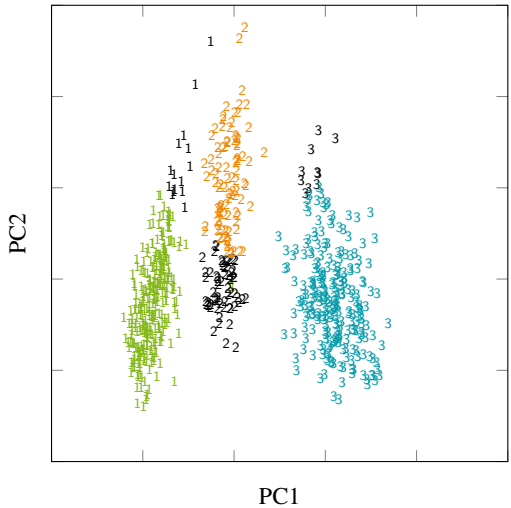
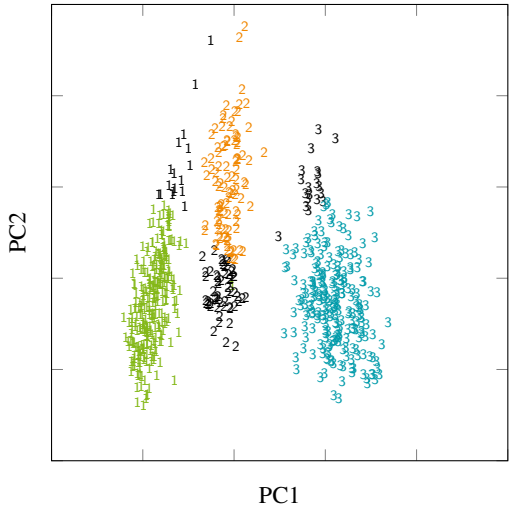
\begin{table*}[btp]\centering
\caption{Clustering results for the scRNA-seq data sets of \citet{Kolodziejczyk/15a} for PAM, PAMSIL, and all variants of PAMMEDSIL. All methods are evaluated for BUILD and Random initialization, and true known $k$=3.}
\label{tab1}
\setlength{\tabcolsep}{6pt}
\begin{tabular}{ l|l|r|r|r|r|r } 
Algorithm & Initialization & AMS & ASW & ARI & NMI & run time (ms) \\
\hline
PAM & BUILD & 0.6566 & 0.5397 & 0.6944 & 0.6549 & 18.26 \\ 
PAM & Random & 0.6566 & 0.5397 & 0.6944 & 0.6549 & 22.67\\ 
PAMMEDSIL & BUILD & \bf0.6747 & 0.5474 & \bf0.7174 & \bf0.6953 & 62.63 \\ 
PAMMEDSIL & Random & \bf0.6747 & 0.5474 & \bf0.7174 & \bf0.6953 & 61.91 \\ 
FastMSC & BUILD & \bf0.6747 & 0.5474 & \bf0.7174 & \bf0.6953 & 25.09 \\
FastMSC & Random & \bf0.6747 & 0.5474 & \bf0.7174 & \bf0.6953 & 24.67 \\
FasterMSC & BUILD & \bf0.6747 & 0.5474 & \bf0.7174 & \bf0.6953 & \bf9.95 \\
FasterMSC & Random & \bf0.6747 & 0.5474 & \bf0.7174 & \bf0.6953 & 10.95 \\
PAMSIL & BUILD & 0.6490 & \bf0.5507 & 0.6962 & 0.6677 & 12493.86 \\
PAMSIL & Random & 0.5799 & 0.5490 & 0.6652 & 0.6633 & 16045.47 \\
\end{tabular}
\end{table*}
We first discuss the results for the data set from \citet{Kolodziejczyk/15a}, shown in Table~\ref{tab1}.
As expected, PAMSIL found the best result with respect to ASW, while the Medoid Silhouette based methods found better results regarding AMS.
PAM, which optimizes the total deviation, found slightly worse results in all measures.
Regarding the known labels,
the highest ARI and NMI scores are achieved by the Medoid Silhouette methods.
The different initializations produced the same results for all methods except PAMSIL here.
Because of the small $k=3$, the speedup of FastMSC over PAMMEDSIL is only small, the additional speedup of FasterMSC is due to reducing the number of iterations.
FasterMSC was able to find the best solutions, but at a 1255$\times$ faster run time than PAMSIL, confirming the expected improvements.
Because AMS and ASW are correlated, with rather small differences between the solutions found by the methods, we argue that AMS is a suitable approximation for ASW, at a much reduced run time.

Since there were no variations in the resulting medoids for the different restarts of the experiment, we can easily compare single results visually.
Figure~\ref{fig52} compares the results of PAMMEDSIL/FastMSC and PAMSIL, showing in black which
points are clustered differently than in the given labels.
As indicated by the similar evaluation scores, both clusterings are very similar, with class~1 captured slightly better in one, class~3 slightly better in the other result.

Table~\ref{tab2} shows the clustering results for the scRNA-seq data sets of \citet{Klein/15a}.
In contrast to Kolodziejczyk et al. data set, we here obtain a slightly higher ARI and NMI for PAMSIL than for the AMS optimization methods.
Again, the results are very similar, with the ASW obtained from the AMS methods being almost identical. The differences in AMS are more pronounced.
Because this data set is larger, the gap in run time becomes more pronounced.
While FasterMSC still finishes in less than a second, PAMMEDSIL now takes multiple seconds to complete, and PAMSIL already needs half an hour to complete.
Because of the larger data set size, FasterMSC has become 16521$\times$ faster than PAMSIL and 6$\times$ faster than PAMMEDSIL.

\begin{table*}[btp]\centering
\caption{Clustering results for the scRNA-seq data sets of \citet{Klein/15a} for PAM, PAMSIL and all variants of PAMMEDSIL. All methods are evaluated for BUILD and Random initialization and true known $k$=4.}
\label{tab2}
\setlength{\tabcolsep}{6pt}
\begin{tabular}{ l|l|r|r|r|r|r } 
Algorithm & Initialization & AMS & ASW & ARI & NMI & run time (ms) \\
\hline
PAM & BUILD & 0.7673 & 0.6825 & 0.8450 & 0.8726 & 355.55 \\ 
PAM & Random & 0.7348 & 0.6292 & 0.8343 & 0.8526 & 476.18\\ 
PAMMEDSIL & BUILD & \bf0.7748 & 0.6834 & 0.8441 & 0.8685 & 2076.15 \\ 
PAMMEDSIL & Random & \bf0.7748 & 0.6834 & 0.8441 & 0.8685 & 3088.77 \\ 
FastMSC & BUILD & \bf0.7748 & 0.6834 & 0.8441 & 0.8685 & 212.01 \\
FastMSC & Random & \bf0.7748 & 0.6834 & 0.8441 & 0.8685 & 305.00 \\
FasterMSC & BUILD & \bf0.7748 & 0.6834 & 0.8441 & 0.8685 & 163.74 \\
FasterMSC & Random & \bf0.7748 & 0.6834 & 0.8441 & 0.8685 & \bf122.63 \\
PAMSIL & BUILD & 0.7649 & \bf0.6838 & \bf0.8483 & \bf0.8739 & 2026025.10 \\
PAMSIL & Random & 0.7220 & 0.6837 & 0.8472 & 0.8724 & 1490354.10 \\
\end{tabular}
\end{table*}

\subsection{Number of Clusters}
\begin{figure*}[tb]\centering
    \begin{subfigure}{.49\textwidth}
	\begin{tikzpicture}[font=\rmfamily\scriptsize]
		\begin{axis}[
		    legend style={at={(.05,.60)},anchor=north west,fill=none,draw=none,inner sep=0,font=\rmfamily\tiny},
		    legend cell align={left},legend columns=2,
			height=23mm,
    	    width=\textwidth-15mm,
    	    scale only axis,
		    every axis label/.style={inner sep=0, outer sep=0},
			xlabel = {number of clusters},
			xmin = 2, xmax = 9,
			ylabel = {Average (Medoid) Silhouette},
			ymin = 0.45, ymax = 0.68,
	    	xtick={2,3,4,5,6,7,8,9}
			]
			\addplot[Pa2, mark=oplus*]coordinates {
			(2, 0.6568)
			(3, 0.6747)
			(4, 0.6675)
			(5, 0.6656)
			(6, 0.6651)
			(7, 0.6044)
			(8, 0.6001)
			(9, 0.5787)
			}; \label{plot_kolod_ams}
   			\addplot[Pa1, mark=oplus]coordinates {
			(2, 0.5232)
			(3, 0.5490)
			(4, 0.5224)
			(5, 0.5013)
			(6, 0.4991)
			(7, 0.4869)
			(8, 0.4776)
			(9, 0.4599)
			}; \label{plot_kolod_asw}
     		\addlegendimage{/pgfplots/refstyle=plot_kolod_ams}\addlegendentry{AMS (FastMSC)}
			\addlegendimage{/pgfplots/refstyle=plot_kolod_asw}\addlegendentry{ASW (PAMSIL)}
		\end{axis}
	\end{tikzpicture}
	\caption{Kolodziejczyk et al. data set contains \textbf{three} different groups of mouse embryonic stem cells.}
	\end{subfigure}
\hfill
 \begin{subfigure}{.49\textwidth}
	\begin{tikzpicture}[font=\rmfamily\scriptsize]
		\begin{axis}[
		    legend style={at={(.05,.25)},anchor=north west,fill=none,draw=none,inner sep=0,font=\rmfamily\tiny},
		    legend cell align={left},legend columns=2,
			height=23mm,
    	    width=\textwidth-15mm,
    	    scale only axis,
		    every axis label/.style={inner sep=0, outer sep=0},
			xlabel = {number of clusters},
			xmin = 2, xmax = 9,
			ylabel = {Average (Medoid) Silhouette},
			ymin = 0.64, ymax = 0.78,
	    	xtick={2,3,4,5,6,7,8,9}
			]
			\addplot[Pa2, mark=oplus*]coordinates {
			(2, 0.7549)
			(3, 0.7701)
			(4, 0.7747)
			(5, 0.7738)
			(6, 0.7532)
			(7, 0.7693)
			(8, 0.7504)
			(9, 0.7529)
			}; 
   			\addplot[Pa1, mark=oplus]coordinates {
			(2, 0.6706)
			(3, 0.6689)
			(4, 0.6838)
			(5, 0.6798)
			(6, 0.6771)
			(7, 0.6595)
			(8, 0.6474)
			(9, 0.6444)
			}; 
		\end{axis}
	\end{tikzpicture}
	\caption{Klein et al. dataset contains embryonic stem cells measured at \textbf{four} different time points.}
	\end{subfigure}
	\caption{Average Medoid Silhouette and Average Silhouette Width for different number of Clusters with FastMSC and PAMSIL.}
	\label{fig6}
\end{figure*}
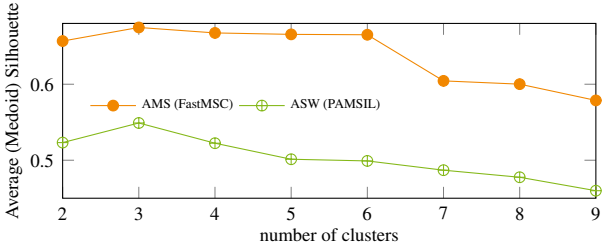
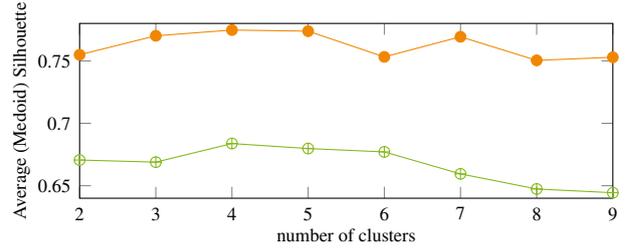
The AMS and ASW are based on very similar ideas of cluster quality, and while AMS is \emph{not a numerical} approximation of the ASW, it usually produces a similar ranking of clusterings. Therefore, if the ASW on a dataset is a suitable heuristic to determine the optimal number, we can assume that the AMS is also a good heuristic for this purpose.

To evaluate how well the Average Medoid Silhouette is suited to choose the optimal number of clusters, we perform FastMSC and PAMSIL on the Klein et al. and Kolodziejczyk et al. data sets for $k = 2$~to~$9$ clusters. Klein et al. data set contains four different groups of data points and Kolodziejczyk et al. data set contains three different groups. As seen in Figure~\ref{fig6}, FastMSC finds the largest AMS as well as PAMSIL the largest ASW at the correct number of clusters for both data sets.

\subsection{Scalability}
To evaluate the scalability of our methods, we use the well-known MNIST data, which has 784 variables ($28 \times 28$ pixels) and 60000 samples. We use the first $N = 1000$ to $30000$ samples and compare $k= 10$ and $k = 100$.
Due to its high run time, PAMSIL is not able to handle this size in a reasonable time.
In addition to the methods for direct AMS optimization, we include the FastPAM1 and FasterPAM algorithms~\cite{DBLP:journals/is/SchubertR21,DBLP:conf/sisap/SchubertR19}.
For all methods we use random initialization.
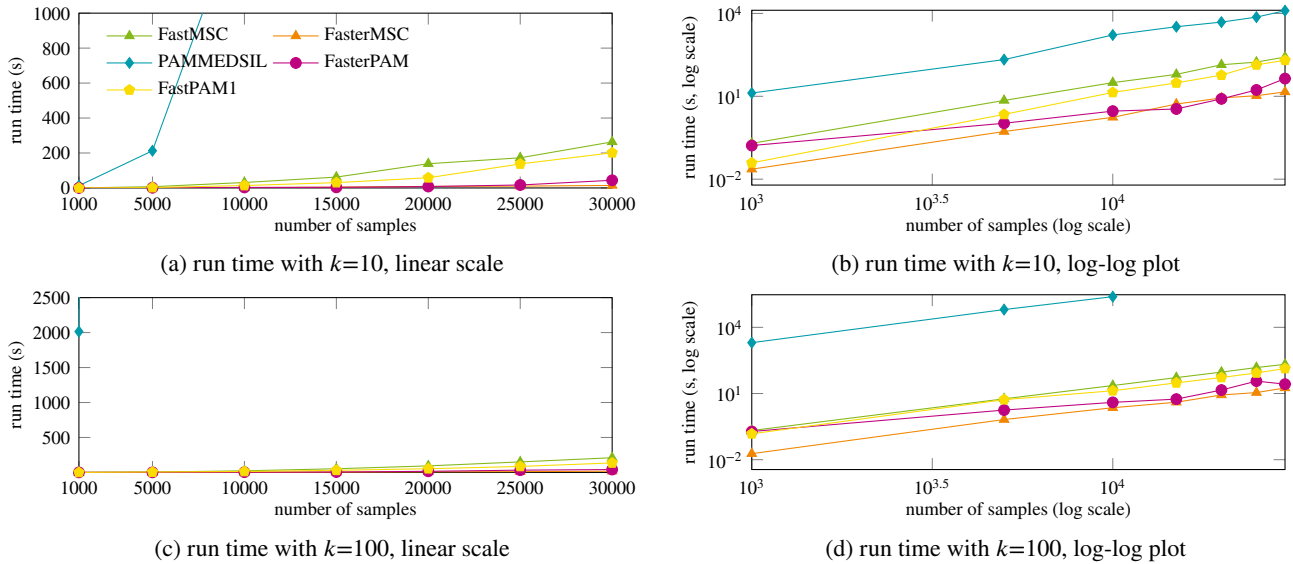
\begin{figure*}[tb]\centering
    \begin{subfigure}{.49\textwidth}
	\begin{tikzpicture}[font=\rmfamily\scriptsize]
		\begin{axis}[
		    legend style={at={(.05,.95)},anchor=north west,fill=none,draw=none,inner sep=0},
		    legend cell align={left},legend columns=2,
			height=23mm,
    	    width=\textwidth-15mm,
    	    scale only axis,
		    every axis label/.style={inner sep=0, outer sep=0},
			xlabel = {number of samples},
			xmin = 1000, xmax = 30000,
			ylabel = {run time (s)},
			ymin = 0, ymax = 1000,
			yticklabel style={/pgf/number format/fixed},
			yticklabel style={/pgf/number format/1000 sep=},
	    	xtick={1000,5000,10000,15000,20000,25000,30000},
			xticklabel style={/pgf/number format/fixed},
			xticklabel style={/pgf/number format/1000 sep=},
			scaled x ticks=false
			]
			\addplot[Pa1, mark=triangle*]coordinates {
			(1000, 0.198)
			(5000, 7.075)
			(10000, 31.187)
			(15000, 62.150)
			(20000, 138.592)
			(25000, 172.462)
			(30000, 263.539)
			};
			\addplot[Pa2, mark=triangle*]coordinates {
			(1000, 0.023)
			(5000, 0.531)
			(10000, 1.734)
			(15000, 5.238)
			(20000, 8.723)
			(25000, 10.502)
			(30000, 14.397)
			};
			\addplot[Pa3, mark=diamond*]coordinates {
			(1000, 13.124)
			(5000, 212.825)
			(10000, 1657.221)
			(15000, 3298.933)
			(20000, 4847.659)
			(25000, 7343.586)
			(30000, 12568.364)
			};
			\addplot[Pa4, mark=oplus*]coordinates {
			(1000, 0.165)
			(5000, 1.056)
			(10000, 2.904)
			(15000, 3.473)
			(20000, 8.076)
			(25000, 16.969)
			(30000, 43.808)
			};
			\addplot[Pa5, mark=pentagon*]coordinates {
			(1000, 0.039)
			(5000, 2.213)
			(10000, 13.783)
			(15000, 30.483)
			(20000, 58.334)
			(25000, 137.264)
			(30000, 201.293)
			};
			\addlegendimage{/pgfplots/refstyle=plot_fpms}\addlegendentry{FastMSC}
			\addlegendimage{/pgfplots/refstyle=plot_fepms}\addlegendentry{FasterMSC}
			\addlegendimage{/pgfplots/refstyle=plot_pms}\addlegendentry{PAMMEDSIL}
			\addlegendimage{/pgfplots/refstyle=plot_fp}\addlegendentry{FasterPAM}
			\addlegendimage{/pgfplots/refstyle=plot_fp1}\addlegendentry{FastPAM1}
		\end{axis}
	\end{tikzpicture}
	\caption{run time with $k{=}10$, linear scale}
	\end{subfigure}
	\hfill
    \begin{subfigure}{.49\textwidth}
	\begin{tikzpicture}[font=\rmfamily\scriptsize]
		\begin{axis}[
		    legend style={at={(0.70,0.35)},anchor=north,fill=none,draw=none,inner sep=0},
		    legend cell align={left},legend columns=2,
    	    scale only axis,
		    every axis label/.style={inner sep=0, outer sep=0},
			height=23mm,
    	    width=\textwidth-15mm,
			xlabel = {number of samples (log scale)},
			xmin = 1000, xmax = 30000,
			ylabel = {run time (s, log scale)},
			ymax = 13000,
			ymode=log, xmode=log,
			yticklabel style={/pgf/number format/fixed},
			yticklabel style={/pgf/number format/1000 sep=},
			xticklabel style={/pgf/number format/fixed},
			xticklabel style={/pgf/number format/1000 sep=},
			scaled x ticks=false
			]
			\addplot[Pa1, mark=triangle*]coordinates {
			(1000, 0.198)
			(5000, 7.075)
			(10000, 31.187)
			(15000, 62.150)
			(20000, 138.592)
			(25000, 172.462)
			(30000, 263.539)
			};
			\addplot[Pa2, mark=triangle*]coordinates {
			(1000, 0.023)
			(5000, 0.531)
			(10000, 1.734)
			(15000, 5.238)
			(20000, 8.723)
			(25000, 10.502)
			(30000, 14.397)
			};
			\addplot[Pa3, mark=diamond*]coordinates {
			(1000, 13.124)
			(5000, 212.825)
			(10000, 1657.221)
			(15000, 3298.933)
			(20000, 4847.659)
			(25000, 7343.586)
			(30000, 12568.364)
			};
			\addplot[Pa4, mark=oplus*]coordinates {
			(1000, 0.165)
			(5000, 1.056)
			(10000, 2.904)
			(15000, 3.473)
			(20000, 8.076)
			(25000, 16.969)
			(30000, 43.808)
			};
			\addplot[Pa5, mark=pentagon*]coordinates {
			(1000, 0.039)
			(5000, 2.213)
			(10000, 13.783)
			(15000, 30.483)
			(20000, 58.334)
			(25000, 137.264)
			(30000, 201.293)
			};
		\end{axis}
	\end{tikzpicture}
	\caption{run time with $k{=}10$, log-log plot}
	\end{subfigure}
	\\
	\begin{subfigure}{.49\textwidth}
	\begin{tikzpicture}[font=\rmfamily\scriptsize]
		\begin{axis}[
		    legend style={at={(0.70,0.35)},anchor=north,fill=none,draw=none,inner sep=0},
		    legend cell align={left},legend columns=2,
    	    scale only axis,
		    every axis label/.style={inner sep=0, outer sep=0},
			height=23mm,
    	    width=\textwidth-15mm,
			xlabel = {number of samples},
			xmin = 1000, xmax = 30000,
			ylabel = {run time (s)},
			ymin = 0, ymax = 2500,
			ytick={500,1000,1500,2000,2500},
			yticklabel style={/pgf/number format/fixed},
			yticklabel style={/pgf/number format/1000 sep=},
			xtick={1000,5000,10000,15000,20000,25000,30000},
			xticklabel style={/pgf/number format/fixed},
			xticklabel style={/pgf/number format/1000 sep=},
			scaled x ticks=false
			]
			\addplot[Pa1, mark=triangle*]coordinates {
			(1000, 0.21)
			(5000, 5.745)
			(10000, 22.975)
			(15000, 52.596)
			(20000, 93.356)
			(25000, 150.420)
			(30000, 210.180)
			};
			\addplot[Pa2, mark=triangle*]coordinates {
			(1000, 0.01879)
			(5000, 0.661)
			(10000, 2.265)
			(15000, 4.025)
			(20000, 8.600)
			(25000, 11.023)
			(30000, 17.910)
			}; 
			\addplot[Pa3, mark=diamond*]coordinates {
			(1000, 2014.386)
			(5000, 63854.893)
			(10000, 245499.702)
			};
			\addplot[Pa4, mark=oplus*]coordinates {
			(1000, 0.313)
			(5000, 1.171)
			(10000, 3.444)
			(15000, 5.547)
			(20000, 15.912)
			(25000, 32.377)
			(30000, 41.221)
			};
			\addplot[Pa5, mark=pentagon*]coordinates {
			(1000, 0.149)
			(5000, 5.254)
			(10000, 13.345)
			(15000, 30.426)
			(20000, 53.143)
			(25000, 85.543)
			(30000, 135.455)
			};
		\end{axis}
	\end{tikzpicture}
	\caption{run time with $k{=}100$, linear scale}
	\end{subfigure}
	\hfill
	\begin{subfigure}{.49\textwidth}
	\begin{tikzpicture}[font=\rmfamily\scriptsize]
		\begin{axis}[
		    legend style={at={(0.70,0.35)},anchor=north,fill=none,draw=none,inner sep=0},
		    legend cell align={left},legend columns=2,
    	    scale only axis,
		    every axis label/.style={inner sep=0, outer sep=0},
			height=23mm,
    	    width=\textwidth-15mm,
			xlabel = {number of samples (log scale)},
			xmin = 1000, xmax = 30000,
			ylabel = {run time (s, log scale)},
			  ymax = 300000,
			ymode=log, xmode=log,
			yticklabel style={/pgf/number format/fixed},
			yticklabel style={/pgf/number format/1000 sep=},
			xticklabel style={/pgf/number format/fixed},
			xticklabel style={/pgf/number format/1000 sep=},
			scaled x ticks=false
			]
			\addplot[Pa1, mark=triangle*]coordinates {
			(1000, 0.21)
			(5000, 5.745)
			(10000, 22.975)
			(15000, 52.596)
			(20000, 93.356)
			(25000, 150.420)
			(30000, 210.180)
			};
			\addplot[Pa2, mark=triangle*]coordinates {
			(1000, 0.01879)
			(5000, 0.661)
			(10000, 2.265)
			(15000, 4.025)
			(20000, 8.600)
			(25000, 11.023)
			(30000, 17.910)
			};
			\addplot[Pa3, mark=diamond*]coordinates {
			(1000, 2014.386)
			(5000, 63854.893)
			(10000, 245499.702)
			};
			\addplot[Pa4, mark=oplus*]coordinates {
			(1000, 0.190)
			(5000, 1.776)
			(10000, 3.976)
			(15000, 5.571)
			(20000, 14.372)
			(25000, 36.438)
			(30000, 26.592)
			};
			\addplot[Pa5, mark=pentagon*]coordinates {
			(1000, 0.149)
			(5000, 5.254)
			(10000, 13.345)
			(15000, 30.426)
			(20000, 53.143)
			(25000, 85.543)
			(30000, 135.455)
			}; 
		\end{axis}
	\end{tikzpicture}
	\caption{run time with $k{=}100$, log-log plot}
	\end{subfigure}
	\caption{Run time on MNIST data (time out 24 hours)}
	\label{fig5}
\end{figure*}
\begin{figure*}[tb]\centering
    \begin{subfigure}{.49\textwidth}
	\begin{tikzpicture}[font=\rmfamily\scriptsize]
		\begin{axis}[
		    legend style={at={(.05,.95)},anchor=north west,fill=none,draw=none,inner sep=0},
		    legend cell align={left},legend columns=2,
			height=23mm,
    	    width=\textwidth-15mm,
    	    scale only axis,
		    every axis label/.style={inner sep=0, outer sep=0},
			xlabel = {number of samples},
			xmin = 1000, xmax = 30000,
			ylabel = {run time (s)},
			ymin = 0, ymax = 700,
			yticklabel style={/pgf/number format/fixed},
			yticklabel style={/pgf/number format/1000 sep=},
	    	xtick={1000,5000,10000,15000,20000,25000,30000},
			xticklabel style={/pgf/number format/fixed},
			xticklabel style={/pgf/number format/1000 sep=},
			scaled x ticks=false
			]
   			\addplot[Pa1, mark=triangle*]coordinates {
			(1000, 0.481)
			(5000, 20.198)
			(10000, 41.631)
			(15000, 74.895)
			(20000, 114.534)
			(25000, 156.904)
			(30000, 292.888)
			}; \label{plot_fastermscd}
			\addplot[Pa2, mark=triangle*]coordinates {
			(1000, 0.353)
			(5000, 11.310)
			(10000, 19.567)
			(15000, 35.949)
			(20000, 58.412)
			(25000, 79.364)
			(30000, 153.298)
			}; \label{plot_dynfastermscd}
			\addplot[Pa3, mark=diamond*]coordinates {
			(1000, 10.491)
			(5000, 39.194)
			(10000, 148.275)
			(15000, 229.779)
			(20000, 337.475)
			(25000, 484.687)
			(30000, 625.753)
			}; \label{plot_sklearnd}
            \addlegendimage{/pgfplots/refstyle=plot_fastermscd}\addlegendentry{FasterMSC}
			\addlegendimage{/pgfplots/refstyle=plot_dynfastermscd}\addlegendentry{DynMSC}
			\addlegendimage{/pgfplots/refstyle=plot_sklearnd}\addlegendentry{scikit-learn k-means}
		\end{axis}
	\end{tikzpicture}
	\caption{run time}
	\end{subfigure}
    \begin{subfigure}{.49\textwidth}
	\begin{tikzpicture}[font=\rmfamily\scriptsize]
		\begin{axis}[
		    legend style={at={(.05,.95)},anchor=north west,fill=none,draw=none,inner sep=0},
		    legend cell align={left},legend columns=2,
			height=23mm,
    	    width=\textwidth-15mm,
    	    scale only axis,
		    every axis label/.style={inner sep=0, outer sep=0},
			xlabel = {number of samples},
			xmin = 1000, xmax = 30000,
			ylabel = {number of swaps},
			ymin = 0, ymax = 3500,
			yticklabel style={/pgf/number format/fixed},
			yticklabel style={/pgf/number format/1000 sep=},
	    	xtick={1000,5000,10000,15000,20000,25000,30000},
			xticklabel style={/pgf/number format/fixed},
			xticklabel style={/pgf/number format/1000 sep=},
			scaled x ticks=false
			]
   			\addplot[Pa1, mark=triangle*]coordinates {
			(1000, 2110.0)
			(5000, 2711.0)
			(10000, 2812.0)
			(15000, 2851.0)
			(20000, 2736.0)
			(25000, 2574.0)
			(30000, 2948.0)
			}; 
			\addplot[Pa2, mark=triangle*]coordinates {
			(1000, 456)
			(5000, 811)
			(10000, 1141)
			(15000, 1498)
			(20000, 1412)
			(25000, 1678)
			(30000, 1704)
			};
		\end{axis}
	\end{tikzpicture}
	\caption{number of swaps for DynMSC and FasterMSC}
	\end{subfigure}
	\caption{Run time and number of swaps on MNIST data for 1000 to 30000 samples. Comparing DynMSC, FasterMSC, and k-means in scikit-learn including Silhouette score calculation. We evaluate $k=2$ to $k=50$, for FasterMSC single calls with random initialization, for scikit-learn single calls with $kmeans++$ initialization with calling the Silhouette score, and DynMSC with maximum $k=50$ and random initialization.}
	\label{fig7}
\end{figure*}
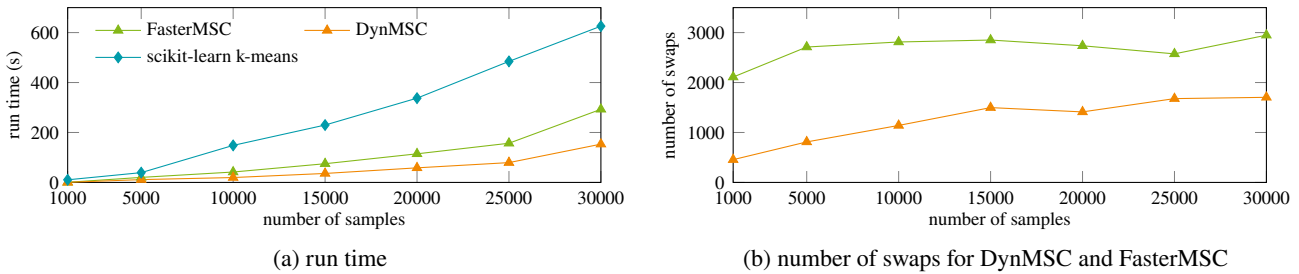
As expected, all methods scale approximately quadratic in the sample size $N$.
FastMSC is on average 50.66$\times$ faster than PAMMEDSIL for $k = 10$ and 10464.23$\times$ faster for $k = 100$,
supporting the expected $O(k^2)$ improvement by removing the nested loop and caching the distances to the nearest centers.
For FasterMSC we achieve even 639.34$\times$ faster run time than for PAMMEDSIL for $k=10$ and 78035.01$\times$ faster run time for $k=100$.
We expect FastPAM1 and FastMSC and also FasterPAM and FasterMSC to have similar scalability;
but since MSC also needs the third nearest neighbor, it needs to maintain more data and access more memory.
We observe that FastPAM1 is 2.50$\times$ faster than FastMSC for $k = 10$ and 1.57$\times$ faster for $k = 100$,
which is larger than expected and due to more iterations necessary for convergence in the MSC methods: FastPAM1 needs on average 14.86 iterations while FastMSC needs 33.48.
In contrast, FasterMSC is even 1.65$\times$ faster than FasterPAM for $k = 10$ and 1.96$\times$ faster for $k = 100$.

To evaluate the scalability of DynMSC, we test it on MNIST for 1000 to 30000 samples for a maximum $k=50$. We compare it to FasterMSC with random initialization run once for each $k=2$ to $k=50$ (keeping the best).
As additional baseline, we include a naive approach using scikit-learn (version 1.2.2) k-means for clustering, and Silhouette only for choosing the number of clusters.
We observe that DynMSC is 1.97$\times$ faster than repeatly running FasterMSC.
This speedup is primarily due to requiring 2.15$\times$ fewer swaps. The difference in the factor between swaps and run time is explained by additional remove operations in DynMSC, which are similarly expensive as the swaps and have to be performed 48$\times$ for the range of $k=2\ldots 50$. DynMSC is on average 9.01$\times$ faster than the popular sklearn routine because of the cost to repeatedly compute the Silhouette.

\section{Conclusions}
We showed that the Average Medoid Silhouette satisfies desirable
theoretical properties for clustering quality measures,
and as an approximation of the Average Silhouette Width
yields desirable results on real problems from gene expression analysis.
We propose a new algorithm for optimizing the Average Medoid Silhouette,
which provides a run time speedup of $O(k^2)$
compared to the earlier PAMMEDSIL algorithm
by caching the nearest centers
and of partial results based on FasterPAM.
This makes clustering by optimizing the Medoid Silhouette possible on much larger data sets than before.
The Medoid Silhouette can also be used to determine the number of clusters in a data set, and the DynMSC algorithm introduced in this article optimizes this process by avoiding redundant computations.
The ability to optimize a variant of the popular Silhouette measure directly
demonstrates the underlying property that any internal cluster evaluation measure specifies a clustering itself, and that the proper unsupervised evaluation of clusterings remains an unsolved problem.
But since \citet{DBLP:journals/sadm/VendraminCH10} found the Silhouette (and its variants) to be among the best and most robust clustering quality criteria, direct optimization of the Medoid Silhouette may be desirable.

Users of cluster analysis are advised to carefully choose the right evaluation measure and clustering method for their problem rather than following a ``default'' recommendation.
Even though the Silhouette appears to be popular and scores high in benchmarks \cite{DBLP:journals/sadm/VendraminCH10}, it may not the best choice for every problem.
In particular Silhouette may perform poorly when (1)~distance is not measured appropriately, (2)~data preprocessing is poor, (3)~clusters have non-convex shape, (4)~clusters vary significantly in diameter, (5)~clusters exist only in subspaces, or (6)~a hierarchy of clustering structures exists.

\section*{Declaration of Competing Interest}

The authors declare that they have no known competing financial interests or personal relationships that could have appeared to influence the work reported in this paper.

\section*{Acknowledgements}

Part of the work on this paper has been supported by the Deutsche Forschungsgemeinschaft (DFG) -– project number 124020371 –- within the Collaborative Research Center SFB 876 ``Providing Information by Resource-Constrained Analysis'', project A2. \url{https://sfb876.tu-dortmund.de/}

At the SISAP 2022 conference, Edgar Chávez suggested to try automatically choosing the number of clusters~$k$, given the run-time improvements of FasterMSC. DynMSC is our newly proposed solution to this challenge that exploits internal data structures of the method.

\printcredits

\bibliographystyle{cas-model2-names}

\bibliography{references}

\end{document}